\documentclass{article} %

\usepackage{iclr2027_conference,times}
\iclrfinalcopy

\usepackage{amsmath,amsfonts,bm}

\def\eqref#1{equation~\ref{#1}}

\def\1{\bm{1}}

\DeclareMathAlphabet{\mathsfit}{\encodingdefault}{\sfdefault}{m}{sl}
\SetMathAlphabet{\mathsfit}{bold}{\encodingdefault}{\sfdefault}{bx}{n}

\usepackage[utf8]{inputenc} %
\usepackage[T1]{fontenc}    %
\usepackage{hyperref}       %
\usepackage{url}            %
\usepackage{booktabs}       %
\usepackage{amsfonts}       %
\usepackage{nicefrac}       %
\usepackage{microtype}      %
\usepackage{xcolor}         %
\usepackage{graphicx}
\usepackage{amsmath,amssymb}
\usepackage{amsthm}
\newtheorem{theorem}{Theorem}
\usepackage{xspace}
\usepackage{svg}
\usepackage{wrapfig}
\usepackage{tabularx}
\usepackage{multirow}
\usepackage{xcolor}
\usepackage{tikz}
\usepackage{subcaption}
\usepackage{xcolor}
\usepackage{etoolbox}

\usepackage{algorithm}
\usepackage{algpseudocode}
\usepackage{placeins}
\usepackage{float}

\newrobustcmd{\redcircle}{\textcolor{red}{\(\bullet\)}}\newcolumntype{Y}{>{\raggedright\arraybackslash}X}

\usepackage{tikz}

\newrobustcmd{\yellowcircle}{%
  \tikz[baseline=-0.6ex]\draw[black, fill=yellow] (0,0) circle (0.45ex);%
}

\definecolor{linkblue}{HTML}{000099} 
\definecolor{muted}{HTML}{1F4E79} %
\definecolor{mydarkblue}{HTML}{101D40}
\definecolor{mydarkpurple}{HTML}{28143D}
\definecolor{mydarkpurple}{HTML}{662E7D}

\hypersetup{
    colorlinks=true,
    urlcolor=mydarkpurple,
    citecolor=mydarkpurple,
    linkcolor=mydarkpurple
}

\title{Imagining Goals with Geometric Horizon \\ Models  for Policy Generalization }

\author{\textbf{Daniel Lawson}\textsuperscript{1,2}, \textbf{Dhruv Sreenivas}\textsuperscript{1,2}, \textbf{Adriana Hugessen}\textsuperscript{1,2}\\ \textbf{Khimya Khetarpal}\textsuperscript{1,3,$\dagger$}, \textbf{Glen Berseth}\textsuperscript{1,2,$\dagger$}\\
$^{1}${Mila},
$^{2}${Université de Montréal},$^{3}${Google DeepMind},
$^\dagger$Equal Advising
}

\newcommand{\pmformat}[2]{%
  \ensuremath{#1\,{\scriptstyle{\pm\,#2}}}%
}

\newcommand{\method}{$\text{TD-Aug}$\xspace}
\newcommand{\methodbf}{$\textbf{TD-Aug}$\xspace}

\newcommand{\methodbase}{$\text{TD-Aug}_{2}$\xspace}
\newcommand{\methodbasebf}{$\textbf{TD-Aug}_{2}$\xspace}

\newcommand{\methodhc}{$\text{TD-Aug}$\xspace}

\newcommand{\methodmc}{$\text{MC-Aug}$\xspace}

\newcommand\blfootnote[1]{%
  \begingroup
    \renewcommand\thefootnote{}%
    \footnotetext{#1}%
    \addtocounter{footnote}{-1}%
  \endgroup
}

\begin{document}

\maketitle
\blfootnote{Correspondence: \texttt{daniel.lawson@mila.quebec}}
\lhead{}  %

\begin{abstract}
Training from offline data has allowed for substantial progress in domains such as robotics, leading to general-purpose policies that can be easily applied zero-shot or efficiently finetuned for downstream tasks. However, these policies can still have poor generalization, both due to the choice of modeling objective and from learning from a static dataset. In this work, we focus on the challenging task of zero-shot goal generalization, where a policy is evaluated on unseen tasks that require combining its existing knowledge (compositional generalization). An avenue for improving a policy's generalization is by generating new experience with world models; however, such generation has proven difficult for longer horizons. Thus, to alleviate this issue, we propose TD-Aug, which samples from a geometric horizon model and allows for directly imagining novel outcomes that can be achieved by composing existing knowledge. We demonstrate that training on these future outcomes as goals for goal-conditioned BC policies significantly improves generalization in stitching-based OGBench tasks.
\\
\textbf{Project page:} \url{https://td-aug.github.io}
\end{abstract}

\section{Introduction}
Training policies with goal-conditioned behavior cloning (BC) provides a simple strategy to learn from prior experience. However, generalization remains a challenge, limiting performance outside of the training data. Generally, world models \citep{ha2018recurrent} present an interesting strategy for improving policies from offline data \citep{hafner2025training}. These models can be used to generate synthetic experience \citep{sutton1991dyna, bruce2024genie}, allowing policies to get more out of limited real data, which may be important in domains such as robotics. In this work, we consider generating experience using a type of world model to imagine counterfactual future outcomes to enable improved goal generalization.

We tackle a gap identified for BC policies showing poor performance on tasks involving stitching \citep{ghugare2024closinggaptdlearning}. In this setting, after policies are trained on a subset of tasks, they are then evaluated zero-shot on harder tasks involving composition. We illustrate an example in Figure \ref{fig:giantsamples}, where a policy is trained on how to solve short horizon tasks (left), and then tested on a compositional task, such as going across the maze (middle).  To solve this problem, we propose to learn a generative model that predicts what states will be encountered after mixing behaviors of the independent data collecting policies (right). Specifically, the generative model we consider is a geometric horizon model (GHM) \citep{janner2020gamma, thakoor22gpighm} trained using Temporal Difference Flows (TD-Flows)  \citep{farebrother2025temporal}. Then, using samples generated from our GHM as a data augmentation for goals, we teach our policy to perform compositional tasks. We provide an overview in Figure \ref{fig:tdaug}, and refer to our method as \methodbf. We further justify this procedure by showing that augmenting BC with goals from the successor measure relates to value weighted regression.

To implement TD-Aug, we study several GHM objectives based on both temporal difference (TD) \citep{farebrother2025temporal, farebrother2026compositionalplanningjumpyworld} and geometric Monte Carlo (MC) \citep{eysenbach2023contrastivelearninggoalconditionedreinforcement} sampling, demonstrating that TD methods are crucial for generalization. Using horizon consistency \citep{farebrother2026compositionalplanningjumpyworld}, we propose a new TD-Flow objective which combines both MC and TD sampling during training, leading to improved long-horizon generation. 

Empirically, in Section \ref{exp:state}, we first validate \method on low-dimensional navigation tasks in OGBench \citep{park2024ogbenchbenchmarkingofflinegoalconditioned} that require stitching for success, substantially improving behavior cloning generalization, from $31\%$ to $79\%$ average success. Then, we demonstrate that our approach can also improve generalization on high-dimensional, image-based tasks ($29\%$ to $56\%$) in Section \ref{sec:visual}  and complicated manipulation tasks ($13\%$ to $52\%$) in Section \ref{sec:manip}.

\begin{figure}[t]
    \centering
    \includegraphics[width=1.0\linewidth]{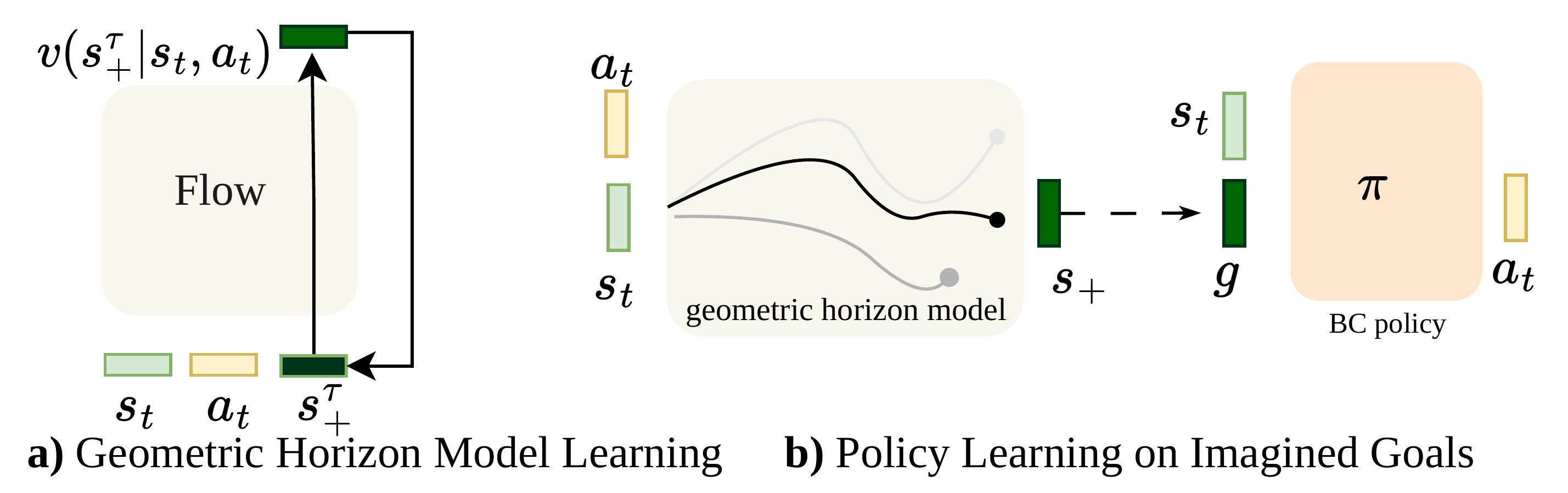}
    \caption{\textbf{Overview of TD-Aug.} We \textbf{a)} train a GHM on an offline dataset to approximate samples from the successor measure of the behavior mixture policy. Then, \textbf{b)} we train BC policies on a mix of ground-truth and generated goals, \textit{improving zero-shot generalization on compositional tasks. }}
    \label{fig:tdaug}
    \vspace{-1em}
\end{figure}

\begin{wrapfigure}[18]{r}{0.42\textwidth}
\vspace{-2ex}
\begin{minipage}{0.41\textwidth}
\begin{algorithm}[H]
\caption{\method}
\label{alg:tdaug}
\begin{algorithmic}
\setlength{\itemsep}{0.35ex}

\vspace{0.5ex}
\Statex \textbf{Representation Learning (Optional) }
\State Use either states $s$, or learned representations $z = \phi(s)$, as the space for GHM and policy  (Section \ref{sec:visual}).

\vspace{0.5ex}
\Statex \textbf{Geometric Horizon Model Training}
\State Train $M^\beta_\phi$ to approximate the successor measure of the behavior mixture policy, using $\mathcal{L}_{\mathrm{MCHC}\text{-}\beta}$ (Equation \ref{eq:td2-mc-hc}).

\vspace{0.5ex}
\Statex \textbf{Policy Training}
\State Train $\pi_\theta$ by BC on a mixture of generated goals from $M^\beta_\phi$ with proportion $\alpha$, as well as real samples (Equation \ref{eq:bcaug}).
\end{algorithmic}
\end{algorithm}
\end{minipage}
\end{wrapfigure}

\section{Related Work}
\textbf{Data augmentation} with temporal structure has been an effective way to re-use experience in reinforcement learning. Models can generate data, for example by sampling transitions for model-free approaches  \citep{sutton1991dyna, ha2018recurrent, feinberg2018model, hafner2020dreamcontrollearningbehaviors, lu2023syntheticexperiencereplay, zhou2023freebellmancompletenesstrajectory}. In the goal-conditioned setting, Hindsight Experience Replay (HER) \citep{andrychowicz2017hindsight} allows re-using transitions with different achieved outcomes as goals \citep{Kaelbling1993LearningTA}. When training policies with BC, sampling hindsight goals is beneficial, but does not cover \textit{compositional
generalization} \citep{ghugare2024closinggaptdlearning}, where a policy is evaluated on state-goal pairs not seen together in the same training trajectories. To mitigate this gap, states are stitched across trajectories using clustering, but this only covers 1-step stitches and relies on heuristics, such as a good distance metric between states. In contrast, we use generative models to capture a future state distribution that directly stitches across trajectories.

\textbf{Long-Horizon Prediction:} In the context of predicting future states, instead of generating states sequentially, prediction can be made more accurate by sampling states in parallel, i.e. jointly sampling sub-trajectories \citep{janner2022planningdiffusionflexiblebehavior}, which could be chained together \citep{luo2025generativetrajectorystitchingdiffusion, lee2025statecoveringtrajectorystitchingdiffusion}, for longer-horizon  or compositional predictions. In comparison, we avoid having to denoise intermediate states or utilize heuristic strategies to chain sub-trajectories together. This is achieved while still generating long-horizon samples in $O(1)$ steps through amortization during GHM training.

\textbf{Representation learning} can be employed to improve  \textit{compositional 
generalization} of BC by encoding temporal dependencies in the representations of a policy, such as with objectives that relate to approximating successor measures or temporal distance \citep{myers2025temporalrepresentationalignmentsuccessor, lawson2025selfpredictive, park2026dualgoalrepresentations}. This has the advantage of avoiding state prediction, but still has challenges, such as with stability and long-horizon generalization. In practice, generative modeling and representation learning can be used together, which is useful for predicting high-dimensional states. 

\textbf{Geometric horizon models (GHMs)}, also known as $\gamma\textrm{-models}$ \citep{janner2020gamma}, propose a generative variant of temporal difference learning, allowing for approximate sampling from the \textit{successor measure} \citep{thakoor22gpighm,tomar2024videooccupancymodels, wiltzer2024distributionalanaloguesuccessorrepresentation, farebrother2025temporal}, a continuous generalization of \textit{successor representations} \citep{dayan1993improving} of a policy. Specifically, we consider Temporal Difference Flows \citep{farebrother2025temporal} (TD-Flows), which use flow-matching objectives for learning such models. Most recent work with TD-Flows \citep{farebrother2025temporal, farebrother2026compositionalplanningjumpyworld} has considered using GHMs for decision-time planning. GHMs can also be used for policy improvement \citep{zheng2026intentionconditionedflowoccupancymodels}, or model-based value expansion \citep{feinberg2018model}, where future synthetic states can be used to estimate bootstrap targets for value estimation \citep{janner2020gamma}. Recent work has used MVE for offline RL using GHMs and Universal Horizon Models (UHMs) \citep{chung2026offlinereinforcementlearninguniversal}, but does not consider the goal-conditioned setting. In contrast, we particularly study GHMs for goal augmentation, and how this alone can induce better zero-shot goal generalization after offline training.

\begin{figure}[t]
    \centering
    \includegraphics[width=1.0\linewidth]{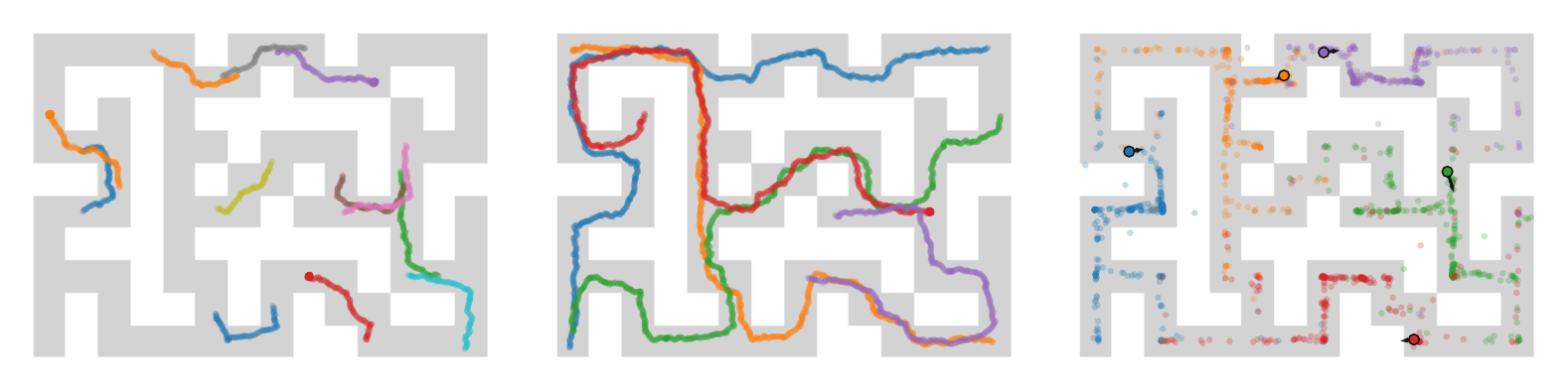}
    \caption{\textbf{Ground-truth and  generated samples.} \textbf{Left}, on \texttt{pointmaze-giant-stitch}, we display ground truth trajectories from the dataset. \textbf{Middle}, we display long-horizon evaluation tasks. \textbf{Right}, we generate samples from an MCHC TD-Flow, as in Section \ref{sec:MCHC}, conditioned on the initial (outlined) state and action. These samples correspond to reachable future states, including those that can only be reached through stitching several trajectories with longer total horizon than short training trajectories. We provide visualizations across more environments in Figure \ref{fig:pointmaze-sample-grid}.}
    \label{fig:giantsamples}
\end{figure}

\section{Background}

\subsection{Goal-Conditioned Reinforcement Learning}

We study offline learning in goal-conditioned Markov Decision Processes (MDPs), defined as the tuple $(\mathcal{S}, \mathcal{A}, \mathcal{G}, P, r, \gamma, \mu)$, where $\mathcal{S}$ is the state space, $\mathcal{A}$ is the action space, $\mathcal{G}$ is the goal space, $P: \mathcal{S} \times \mathcal{A} \rightarrow \Delta(\mathcal{S})$ is the transition kernel, $r$ is the reward function, $\gamma \in [0,1)$ is the discount, and $\mu \in \Delta(\mathcal{S})$ is an initial state distribution. We consider that goals coincide with states ($\mathcal{G}=\mathcal{S}$), and the reward is the Dirac delta at the goal: $r_{g}(s_t,a_t,s_{t+1}) = \delta_{g}(s_{t+1})$. We aim to learn a policy $\pi: \mathcal{S} \times \mathcal{G} \rightarrow \Delta(\mathcal{A})$ using a dataset $\mathcal{D}$, \emph{without online interaction}, to maximize expected return:
\begin{align*}
    J(\pi) = \mathbb{E}_{g \sim p(g), s_0 \sim \mu,a_t \sim \pi(\cdot \mid s_t, g),
s_{t+1} \sim P(\cdot \mid s_t,a_t)
} \left[ \sum_{t=0}^\infty \gamma^t r_g(s_t, a_t, s_{t+1}) \right]
\end{align*}

For the subset $X \subseteq \mathcal{S}$, we define the \emph{normalized successor measure} of a policy $\pi$ as:
\begin{align}
    M^\pi_\gamma(X \mid s, a) \triangleq (1-\gamma) \sum_{k \geq 0} \gamma^k \mathrm{Pr}(s_{k+1} \in X \mid s_0 = s, a_0 = a, \pi)
\end{align}
where $\mathrm{Pr}$ denotes the time-specific probability that we are in $X$ starting from $(s, a)$ when following $\pi$. We write the density of the successor measure $m^\pi_\gamma$ via $\smash{M^\pi_\gamma(\mathrm{d}s_+\mid s,a) = m^\pi_\gamma(s_+\mid s,a)\,\mathrm{d}s_+}$, and the action-marginalized successor measure is defined as $\smash{M^\pi_\gamma(\cdot\mid s)\triangleq {\mathbb{E}}_{a \sim\pi(\cdot\mid s)}[M^\pi_\gamma(\cdot\mid s,a)]}$. We also define the discounted state occupancy density as the marginal of $\smash{m^\pi_\gamma}$ over the initial state distribution: $\smash{d^\pi(s)\triangleq
{\mathbb{E}}_{s_0\sim\mu, a_0\sim\pi(\cdot\mid s_0)}[m^\pi_\gamma(s\mid s_0,a_0)]}$.
With the goal reaching reward, the successor measure's density relates to the \textit{goal-conditioned value function} via $    Q^\pi_\gamma(s,a, g)
    = (1-\gamma)^{-1} m^\pi_\gamma(g \mid s,a)$.

\subsection{Flow Matching}
Following success in modeling policies and geometric horizon models \citep{farebrother2025temporal}, we utilize \textit{flow matching} \citep{lipman2023flow}. With a \emph{source} distribution $p_0 \in \Delta(\mathbb{R}^d)$ and a \emph{target} distribution $p_1 \in \Delta(\mathbb{R}^d)$, flow matching learns a velocity field $v^\tau: \mathbb{R}^d \rightarrow \mathbb{R}^d$ such that for $x_0 \sim p_0$, the sample $x_1 = x_0 + \smash{\int_0^1} v^\tau (x_\tau) \mathrm{d}\tau$ matches the target distribution ($x_1 \sim p_1$). 

In practice, we use \emph{conditional} flow matching \citep{lipman2023flow}:
\begin{align*}
    \mathcal{L}_{\textrm{CFM}}(\theta) = \mathbb{E}_{x_1 \sim \mathcal{D}, \tau \sim U(0, 1), x_\tau \sim p_{\tau \mid 1}(\cdot \mid x_1)} \left[ \lVert v^\tau_\theta(x_\tau ) - u_\tau(x_\tau\mid x_1) \rVert_2^2 \right].
\end{align*}

We use the \emph{conditional optimal transport} (CondOT) coupling, drawing a straight line  between $x_0$ and $x_1$, $ x_\tau = (1 - \tau) x_0 + \tau x_1$. This gives the true conditional velocity field target as $u_\tau(x_\tau \mid x_1) = x_1 - x_0$. Our source distribution $p_0$ is the standard Gaussian, $\mathcal{N}(0,I)$.

\subsection{Training Policies}
\label{sec:train_pol}
We train policies on datasets $\mathcal{D} = \{(s_0^i, a_0^i, \cdots , s_T^i, a_T^i)\}_{i=1}^N$ comprised of trajectories collected by a set of unknown policies $\{\beta_j(a | s)\}$. These policies can be marginalized by a mixture policy $\beta(a~|~s) = \sum_j \beta_j(a ~|~s) p(\beta_j ~|~s)$, which is a Markovian policy that has occupancy equal to the expectation under the individual policies: $\smash{d^\beta(s) = \mathbb{E}_{p(\beta_j)} \left[d^{\beta_j}(s)\right]}$ \citep{ghugare2024closinggaptdlearning}. In general, there is not equality in the conditional occupancy (the SM), $M^\beta(\cdot | s,a) \neq \mathbb{E}_{\beta_j \sim p(\beta_j | s,a)}\left[M^{\beta_j}(\cdot | s,a)\right]$. 

We consider the standard \textit{goal-conditioned BC} objective, which samples a current state from the dataset ($d^\beta)$, and trains the policy to reach goals sampled as future states in the same trajectory, corresponding to samples from the successor measure of an individual policy $ M^{\beta_j}$:
\begin{align}
\label{eq:bc}
\mathcal{L}_{\text{BC}}({\theta}) = -\mathbb{E}_{\substack{s\sim d^{\beta} ,  a \sim \beta(\cdot~|~s) \\ \beta_j \sim p(\beta_j |s,a), ~s_+ \sim M^{\beta_j}(\cdot | s,a) }} \left[\log {\pi_\theta}(a | s, g=s_+)\right]
\end{align}
We perform goal-conditioned behavior cloning using flow matching, representing the policy by a conditional velocity field $v^\tau_\theta (a_\tau \mid s,g)$:
\begin{align}
\mathcal{L}_{\text{BC}}(\theta) = \mathbb{E}_{(s, a, s_+), \tau \sim U(0,1), a_0 \sim p_0,  a_\tau =(1-\tau)a_0 + \tau a}\left[\|v^\tau_\theta(a_\tau \mid s, g=s_+) - (a - a_0) \|_2^2\right]
\end{align}
To sample from $\pi_\theta(\cdot | s,g)$, using $a_0 \sim p_0$, we integrate with $v_\theta^\tau$, via $a_1 = a_0 + \int_0^1 v^\tau_\theta(a_\tau |s,g)\mathrm{d}\tau$.
\subsection{Temporal Difference Flows}

Recently, there has been significant interest in developing generative models of the successor measure, called geometric horizon models (GHMs), that can sample states from the \emph{geometric horizon} of any given policy \citep{janner2020gamma, thakoor22gpighm}. These models can be trained via TD learning, as for all $(s, a)$, the successor measure $M^\pi_\gamma$ satisfies:
\begin{align}
\label{tdsm}
M^\pi_\gamma(\cdot | s,a) = (1-\gamma) P(\cdot | s,a) + \gamma\mathbb{E}_{s'\sim P(\cdot | s,a), a' \sim \pi(\cdot | s')}[M^\pi_\gamma(\cdot | s', a')]
\end{align}
Optimizing this objective is challenging with prior generative modeling techniques such as GANs \citep{gans} or normalizing flows \citep{rezende2015normflows}, especially at large discounts $\gamma$. This led to the proposal to use flow matching to learn GHMs \citep{farebrother2025temporal}, which showed strong success in simulated control domains. TD-Flows parameterize the GHM as a velocity field $\smash{v^\tau_\phi}: \mathbb{R}^d \times \mathcal{S} \times \mathcal{A} \rightarrow \mathbb{R}^d$ with parameters $\phi$, such that with source $x_0 \sim p_0 \in \Delta(\mathbb{R}^d)$, the velocity field induces the flow $\psi^\tau: \mathbb{R}^d \times \mathcal{S} \times \mathcal{A} \rightarrow \mathbb{R}^d$:
\begin{align}
\psi^\tau_\phi(x_0 | s,a) = x_0 + \int_0^\tau v^{i}_\phi(\psi^i_\phi(x_0 | s,a) | s,a)\mathrm{d}i.
\end{align}
We aim to learn $v^\tau_\phi$, such that for $x_0 \sim p_0$, the samples $\psi^1_\phi(x_0 | s, a)$ are distributed according to $\smash{M^\pi_\gamma(\cdot | s, a)}$. Thus, we will use the notation $\smash{M^\pi_\phi(\cdot | s, a)}$ for our  approximate successor measure parameterized by $\smash{v^\tau_\phi}$. 
Exploiting the Bellman structure of the successor measure, \cite{farebrother2025temporal} show that a flow that models said measure can be learned by minimizing the $\text{TD}^2\text{-CFM}$ loss:
\begin{align}
\label{eq:td2}
&\mathcal{L}_{\text{TD}^2}(\phi ; \bar\phi) = (1 - \gamma) \mathcal{L}_{\textrm{td-ss}}(\phi) + \gamma\mathcal{L}_{\textrm{td-bs}}(\phi ; \bar\phi),\text{ where:}\\
&\mathcal{L}_{\textrm{td-ss}}(\phi)=\mathbb{E}_{\tau\sim U(0,1), (s,a,s')\sim \mathcal{D}^\beta, x_0 \sim p_0, x_\tau=(1-\tau)x_0 + \tau s'}\left[\|v^\tau_\phi(x_\tau |s,a) - (s' - x_0)\|^2\right]\notag\\
&\mathcal{L}_{\textrm{td-bs}}(\phi)=\mathbb{E}_{\tau\sim U(0,1), (s,a,s')\sim \mathcal{D}^\beta, a' \sim \pi(\cdot | s'), x_0 \sim p_0, x_\tau = \psi^\tau_{\bar \phi}(x_0 | s',a')}\left[\|v^\tau_\phi(x_\tau |s,a) - v^\tau_{\bar\phi}(x_\tau |s',a')\|^2\right]\notag
\end{align}
where $\bar\phi$ are target flow parameters, as commonly used in value-based RL \citep{mnih2015humanlevel}. Recalling \eqref{tdsm}, $\mathcal{L}_{\textrm{td-ss}}$ targets the single-step transition kernel, while $\mathcal{L}_{\textrm{td-bs}}$ bootstraps using the successor measure at the next state. Together, this leads to approximating $\smash{M^\pi_\gamma}$.

We can also approximate the successor measure $\smash{M^\beta(\cdot | s,a)}$ of the mixture policy $\beta$, rather than an arbitrary policy $\pi$. While this requires sampling $a'\sim\beta(\cdot| s')$ in the Bellman update, in practice, we will generally bootstrap with the next action in the dataset, yielding the term:
$$\mathcal{L}_{\textrm{td-bs}}(\phi)=\mathbb{E}_{\tau\sim U(0,1), (s,a,s',a')\sim \mathcal{D}^\beta,x_0 \sim p_0, x_\tau = \psi^\tau_{\bar \phi}(x_0 | s',a')}\left[\|v^\tau_\phi(x_\tau |s,a) - v^\tau_{\bar\phi}(x_\tau |s',a')\|^2\right]$$
We refer to this objective as $\mathcal{L}_{\text{TD}^2\text{-}\beta}$. This is a simpler objective than the fully off-policy objective in \eqref{eq:td2}, as we bootstrap over fully in-support actions.  If one does not need to stitch across trajectories, we can utilize a simpler approach to model $M^\beta_\gamma$ using  empirical Monte Carlo samples of future states. This allows us to learn a generative model directly with supervised learning:
\begin{align}
    \label{eq:mc-cfm}
    \mathcal{L}_\textrm{MC}(\phi) = \mathbb{E}_{\substack{(s, a) \sim \mathcal{D}^\beta, k\sim \mathrm{geom}(1-\gamma), s_+ = s_{\rightarrow k} \\ \tau \sim U(0, 1), x_0 \sim p_0, x_\tau=(1-\tau)x_0 + \tau s_+}} \left[ \lVert v^\tau_\phi (x_\tau \mid s, a) - (s_+ - x_0) \rVert_2^2 \right].
\end{align}
Here, $s_{\rightarrow k}$ denotes the state that is $k$ steps further from $s$ along the sampled trajectory in the dataset.
If we had infinitely long trajectories, collected directly by the mixture policy, this objective would approximate $M^\beta(\cdot | s,a)$. However, in practice, we have datasets with truncated horizon, and collected from a set of data-collecting policies. This leads $\mathcal{L}_\textrm{MC}(\phi)$ to instead approximate a mixture of successor measures $\smash{\mathbb{E}_{\beta_j \sim p(\beta_j | s,a)} {M^{\beta_j}_\gamma}(\cdot | s,a)}$ \citep{lawson2025selfpredictive}, preferring the use of $\mathcal{L}_{\text{TD}^2\text{-}\beta}$.

\section{Improving Long-Horizon, On-Policy Generation by Combining MC and TD Sampling}
\label{sec:MCHC}
In this section, we propose a variant of the TD-Flow training procedure ($\mathcal{L}_{\text{TD}^2\text{-}\beta}$) to address scaling to long horizon prediction, which is affected by the discount $\gamma$ used during bootstrapping. Intuitively, fitting the GHM is easier at lower discounts, as the loss is more heavily weighted to the ground truth next-state prediction, rather than the bootstrap. Recent work proposes the  Temporal Difference Horizon Consistency (TD-HC) objective \citep{farebrother2026compositionalplanningjumpyworld}, which proposes how to use a "local horizon" GHM (e.g. a GHM conditioned on a discount $\textcolor{orange}{\omega} < \gamma$)  to obtain an estimate of the successor measure to approximate a  higher discount GHM ($\textcolor{blue}{\gamma}$) \citep{janner2020gamma, farebrother2026compositionalplanningjumpyworld}, i.e. $\textcolor{orange}{\omega} < \textcolor{blue}{\gamma}$ through a Bellman-like relation:
\begin{align}
\label{eq:bellmantdhc}
M^\pi_{\textcolor{blue}{\gamma}}(\cdot | s,a) \triangleq& (1-\textcolor{blue}{\gamma}) P(\cdot | s,a) + \textcolor{blue}{\gamma} \frac{1 - \textcolor{blue}{\gamma}}{1 - \textcolor{orange}{\omega}}\mathbb{E}_{s'\sim P(\cdot | s,a), a' \sim \pi(\cdot | s')}[M^\pi_{\textcolor{orange}{\omega}}
(\cdot | s', a')] \notag \\
+ &\textcolor{blue}{\gamma}\frac{\textcolor{blue}{\gamma}- \textcolor{orange}{\omega}}{1 - \textcolor{orange}{\omega}}\mathbb{E}_{s'\sim P(\cdot | s,a), a' \sim \pi(\cdot | s'), s_+ \sim M^\pi_{\textcolor{orange}{\omega}}
(\cdot | s', a'), a_+ \sim \pi(\cdot | s_+)}[M^\pi_{\textcolor{blue}{\gamma}}
(\cdot | s_+, a_+)] 
\end{align}
This can be converted to a TD-Flow loss, $\mathcal{L}_{\mathrm{TDHC}}$, by modifying \eqref{eq:td2}, shown in Appendix \ref{sec:TD-HC}.

We propose a variant of $\mathcal{L}_{\mathrm{TDHC}}$ adapted for the on-policy setting to approximate $M^\beta_{\textcolor{blue}{\gamma}}$. While \eqref{eq:mc-cfm} supervises the flow with ground-truth samples, which we will show leads to high-quality generations, it does not generate samples that can be reached through stitching trajectories. In contrast, \eqref{eq:td2} stitches, but uses sparse ground-truth signal, which decreases as the discount is increased. This may make it harder to have higher-quality samples and to scale to longer horizons. To obtain the best of both objectives, we consider a \textit{horizon-consistent objective utilizing MC samples} from the dataset. While we will learn $M^\beta_{\textcolor{blue}{\gamma}}$, we will not  learn $M^\beta_{\textcolor{orange}{\omega}}$, and instead utilize MC samples from dataset trajectories. This leads to the objective:
\begin{align}
\label{eq:td2-mc-hc}
&\mathcal{L}_{\text{MCHC-}\beta}(\phi ; \bar\phi) = (1 - \textcolor{blue}{\gamma}) \mathcal{L}_{\textrm{td-ss}}(\phi) +  \textcolor{blue}{\gamma}\frac{1 - \textcolor{blue}{\gamma}}{1 - \textcolor{orange}{\omega}}\mathcal{L}_{\textrm{mc}}(\phi) +\textcolor{blue}{\gamma}\frac{\textcolor{blue}{\gamma}- \textcolor{orange}{\omega}}{1 - \textcolor{orange}{\omega}}\mathcal{L}_{\textrm{td-mc-bs}}(\phi ; \bar\phi),\text{ where:}\\
&\mathcal{L}_{\textrm{td-ss}}(\phi)=\mathbb{E}_{\tau\sim U(0,1), (s,a,s')\sim \mathcal{D}^\beta, x_0 \sim p_0, x_\tau=(1-\tau)x_0 + \tau s'}\left[\|v^\tau_\phi(x_\tau |s,a) - (s' - x_0)\|^2\right]\notag\\
&\mathcal{L}_{\textrm{mc}}(\phi)=\mathbb{E}_{\substack{\tau\sim U(0,1), (s,a)\sim \mathcal{D}^\beta, k \sim \mathrm{geom}(1-\textcolor{orange}{\omega}) \\ s_+=s_{\rightarrow k+1}, x_0 \sim p_0, x_\tau=(1-\tau)x_0 + \tau s_+}}\left[\|v^\tau_\phi(x_\tau |s,a) - (s_+ - x_0)\|^2\right]\notag\\
&\mathcal{L}_{\textrm{td-mc-bs}}(\phi)=\mathbb{E}_{\substack{\tau\sim U(0,1), (s,a)\sim \mathcal{D}^\beta ,k \sim \mathrm{geom}(1-\textcolor{orange}{\omega}) \\ s_+,a_+=s_{\rightarrow k+1},a_{\rightarrow k +1}, x_0 \sim p_0, x_\tau = \psi^\tau_{\bar \phi}(x_0 | s_+,a_+)}}\left[\|v^\tau_\phi(x_\tau |s,a) - v^\tau_{\bar\phi}(x_\tau |s_+,a_+)\|^2\right]\ \notag
\end{align}

This loss is not just supervised by the next ground truth states, but also future MC samples. Additionally, instead of bootstrapping from the next state $s_{t+1}$, we will bootstrap from this MC sample, which is obtained as a future state after $s_{t+1}$, i.e. $s_+ = s_{t+k+1}$, where $k \sim \text{geom}(1 - \textcolor{orange}{\omega})$. As we utilize MC samples, which are both truncated, and are collected by a single policy, they do not correspond to the true $M^\beta_{\textcolor{orange}{\omega}}$. Thus, this objective is a biased approximation of $M^\beta_{\textcolor{blue}{\gamma}}$ as $\textcolor{orange}{\omega}$ increases, but allows increasing $\textcolor{blue}{\gamma}$ while still training on higher proportions of ground-truth data.

\section{Imagining Goals for Policy Generalization}
\subsection{Generalization Problem} 
We consider a generalization problem for BC policies. Following the setting from Section \ref{sec:train_pol}, we train on trajectories collected by unknown policies $\{\beta_j(a | s)\}$, with mixture policy $\beta(a|s)$. We will now define training and test distributions for $\pi_\theta$
\eqref{eq:bc}, following \citet{ghugare2024closinggaptdlearning}:
\begin{align}
\label{eq:gap}
R^\pi_{\text{id}} = \underbrace{\mathbb{E}_{\substack{s \sim d^{\beta},\beta_j \sim p(\beta_j | s),s_g \sim M^{\beta_j}(\cdot | s)}}
\bigl[R^{{\pi_\theta}}(s, s_g)\bigr]}_{\text{ in-distribution training tasks  }},R^\pi_{\text{gen}}=\underbrace{\mathbb{E}_{\substack{s \sim d^\beta,s_g \sim M^{\beta}(\cdot| s)}}
\bigl[R^{{\pi_\theta}}(s, s_g )\bigr]}_{\text{tasks requiring compositional generalization}}
\end{align}
With $R^\pi_\theta(s, s_g)$ defining the success rate of the policy to reach $s_g$ from $s$,    $R^\pi_{\text{id}}$ averages performance over in-distribution tasks and $R^\pi_{\text{gen}}$ describes average success on a class of out-of-distribution tasks requiring compositional generalization. These include tasks which may not be completed by a single data collecting policy $\beta_j$, but still have sufficient data coverage to complete through composing or stitching trajectories collected by different policies, i.e. reachable with $\beta$. 

While \eqref{eq:bc} corresponds to improving on $R^\pi_{\text{id}}$, the BC objective does not cover $R^\pi_{\text{gen}}$, leading to poor empirical performance. We can see that if we had access to goals coming from the successor measure of the mixture policy $\beta$ during training, rather than that of individual policies, our policy could learn to reach goals required for better $R^\pi_{\text{gen}}$. 

\subsection{Method} 
To address this type of policy generalization, we can directly approximate and sample from the target generalization distribution. 
We note that because we do not have full rollouts from the mixture policy, we must utilize TD learning to stitch, relying on $\mathcal{L}_{\text{TD}^2\text{-}\beta}$ or $\mathcal{L}_{\text{MCHC}\text{-}\beta}$, rather than $\mathcal{L}_\textrm{MC}(\phi)$. 
Thus, we train \textbf{on-policy TD-Flows}, which will yield an approximation $ \smash{M^\beta_\phi(\cdot| s,a)}$ for the successor measure of the mixture policy. After obtaining $\smash{M^\beta_\phi}$, we perform \textbf{goal augmentation} where  we train on a mix of original goals plus samples from the generative model, covering the compositional distribution:
\begin{align}
\label{eq:bcaug}
 &\max_{\pi_\theta}\mathbb{E}_{\substack{s\sim d^{\beta}, a \sim \beta(a|s), ~s_+ \sim P_{\text{aug}}(\cdot | s,a)}} \left[\log {\pi_\theta}(a | s, g=s_+)\right] \\
 \label{eq:paug}
 & \mathrm{where~}P_{\text{aug}}(\cdot | s,a) \triangleq \underbrace{\alpha{ {M}^{\beta}_\phi(\cdot | s,a) }}_{\text{imagined compositional samples} }
 + \underbrace{(1-\alpha) \mathbb{E}_{p(\beta_j | s,a) }{M^{\beta_j}(\cdot | s,a) }}_{\text{ground truth future states}}
\end{align}

We generally refer our goal augmentation procedure as \methodbf, and by default use  $\mathcal{L}_{\text{MCHC-}\beta}$. We can also instantiate our method with $\mathcal{L}_{\text{TD}^2\text{-}\beta}$, which we denote later with \methodbasebf .  We provide an overview of our method in Algorithm \ref{alg:tdaug} and Figure \ref{fig:tdaug}.
An advantage of this procedure is that we are not limited to a fixed number of stitches, or heuristic sampling strategies \citep{ghugare2024closinggaptdlearning}. We also perform $O(1)$ sampling steps per sample to cover an expected horizon $H = (1 - \gamma)^{-1}$, while one-step models need $O(H)$ steps, potentially resulting in compounding error and longer time for rollouts.

We relate this procedure to offline goal-conditioned RL, specifically 1-step policy improvement over the mixture policy with value-weighted regression~\citep{peters2007reinforcement, wang2018exponentially, peng2019advantage,wang2020critic}, for a brief proof, see Appendix~\ref{theorem:one-step-proof}: 
\begin{theorem}[Equivalence to 1-step improvement]
\label{thm:one-step-rwr}
With $\alpha = 1$, and exact samples from $M^\beta(\cdot | s,a)$, our augmentation strategy in \eqref{eq:bcaug} is equivalent to value-weighted regression used for one-step improvement over behavior policy $\beta$:
\begin{align}
\theta^* \in \arg \max_{\theta} \;
\mathbb{E}_{\substack{s \sim d^{\beta}, s_+ \sim d^{\beta} \\ a \sim \beta(\cdot \mid s)}}
\left[
\frac{Q^\beta(s,a, g =s_+)}{d^\beta(s_+)}
\,\log \pi_\theta(a \mid s, g=s_+)
\right]
\end{align}
\end{theorem}

\section{Experiments}

\begin{wrapfigure}{r}{0.5\linewidth}
    \centering
    \vspace{-3em}
    \includegraphics[width=\linewidth]{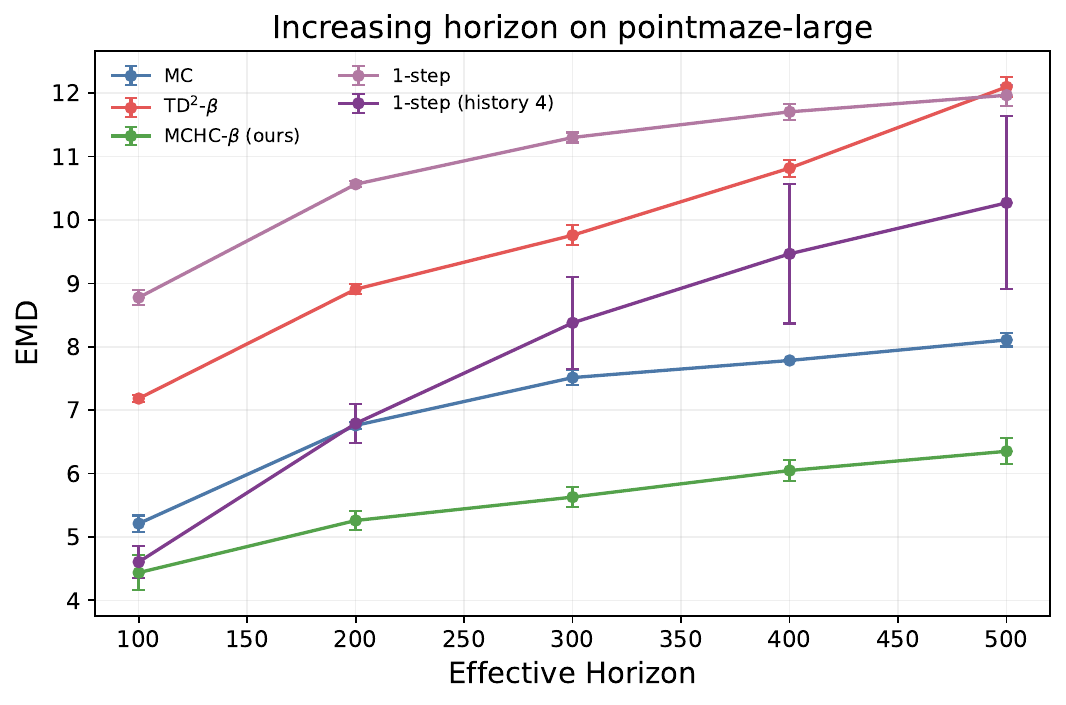}
    \caption{\textbf{Evaluating sampling with increasing horizon}. We compute earth mover's distance (EMD) between empirical samples and generated samples on held-out evaluation data for \texttt{pointmaze-large-stitch}.}
    \label{fig:reload_emd_xy}
    \vspace{-3em}
\end{wrapfigure}

In this section, we show that \methodbf can improve combinatorial generalization across a variety of \texttt{stitch} tasks in the OGBench \citep{park2024ogbenchbenchmarkingofflinegoalconditioned} benchmark suite. In Section \ref{exp:state}, we (a) evaluate the quality of generated samples as horizon scales for different GHM objectives on low-dimensional navigation tasks,  and (b) show that \method is useful for improving BC generalization, even on long-horizon tasks requiring many stitches. Then, (c) we show that our method can scale to visual environments by learning GHMs in latent space in Section \ref{sec:visual}, and Section \ref{sec:manip}, we demonstrate that we can also improve stitching on manipulation tasks.

\subsection{State-based Environments}
\label{exp:state}

We first evaluate our method on $\texttt{pointmaze-stitch}$ navigation tasks from OGBench, which allow us to isolate the challenge of long-horizon stitching from high-dimensional dynamics. In these tasks, agents are trained on short trajectories, which travel up to 4 maze cells. Then, agents are evaluated on longer compositional tasks such as those that span the whole width of the maze. We  find that the original stitching datasets, especially for point, are not terminated upon goal-reaching, resulting in many transitions that idle at the goal. Thus, we re-collect the stitch datasets that terminate upon reaching a sub-goal. On new datasets, baselines perform similarly, or slightly better than before.

To begin our study, we first \textbf{compare the quality of generative models}. \textit{Qualitatively}, we provide visualizations on \texttt{pointmaze-\{medium,large,teleport,giant\}-stitch} in Figure \ref{fig:pointmaze-sample-grid}. We can see that MC Flows can accurately capture training trajectories and multi-modality in the dataset. We can see this captures both multi-modality in the policy, and environment stochasticity, as shown in \texttt{teleport}, where upon entering a black hole, the agent is randomly sent to a random white hole.  However, the MC Flow does not implicitly stitch across trajectories. On the other hand, the TD-based objectives (${\text{TD}^2\text{-}\beta}, {\text{MCHC}\text{-}\beta}$), can generate samples that require stitching, capturing states reachable beyond than the training trajectory horizon. However, the $\text{TD}^2\text{-}\beta$ samples are more diffuse,  or may go backwards from the direction implied by the ground truth samples. In contrast, the $\text{MCHC}\text{-}\beta$ samples balance implicit stitching with the quality of MC samples.

Next, we perform a \textit{quantitative} comparison between model samples on a held-out evaluation dataset on \texttt{pointmaze-large-stitch}. From an initial state, we can see how models (trained on short trajectories) generate samples that are closer in distribution to longer, held-out trajectories, measuring stitching. This is performed by sampling initial (state,action) tuples from a held-out dataset. Then, from these tuples we sample future states geometrically, obtaining empirical samples. From these same initial states, we sample from our model, as shown in Figure \ref{fig:largediscountsweep}. Then, we compute the earth mover's distance (EMD) between these two sets of points, and then average across different initial positions. We repeat this for models trained at various effective horizons $H=(1-\gamma)^{-1}$, measuring against ground truth samples at the corresponding horizon, as shown in Figure \ref{fig:reload_emd_xy}. 
In addition to the GHM based models, we also generate samples from autoregressive one-step prediction, which we further discuss in Appendix \ref{sec:one-step}. 
While some models have similar EMD at the initial horizon (100), we can see that the errors of $\text{TD}^2\text{-}\beta$, and 1-step baselines, quickly increases with horizon. Additionally, the EMD of MC increases because it is trained on short trajectories, and thus has no mechanism to stitch. In contrast, $\text{MCHC}\text{-}\beta$ has the lowest EMD as horizon increases, showing the benefit of mixing both MC targets and TD when training on-policy GHMs.

\begin{table}[t]\centering
\tabcolsep=3pt   
\cmidrulewidth=0.5pt
\fontsize{7}{7}\selectfont
\caption{\textbf{Generalization through compositional data augmentation.} With goal conditioned flow behavior cloning (GCFBC), we compare training on the original dataset to various data augmentation strategies. We calculate success rate averaged over the last 3 checkpoints, and then report the mean, std over 3 seeds. We \textbf{bold} the top performing methods. }
\label{tab:initial}
\begin{tabular}{lllllll|llll}\toprule

{\centering \texttt{pointmaze}} 
& {\centering GCFBC} 
& {\shortstack{+Cluster-Aug}} 
& {\shortstack{+Dyn-Aug}}
& {\shortstack{+\methodmc}} 
& {\shortstack{+\methodbase}} 
& {\shortstack{+\methodhc}}
& \text{GCIQL}
& \text{QRL}
& \text{CRL}
& \text{HIQL}
\\\midrule

\texttt{medium-stitch} 
& \pmformat{55}{10} 
& \pmformat{89}{5} 
& \pmformat{94}{6}
& \pmformat{70}{2} 
& \pmformat{96}{1} 
& \pmformat{\mathbf{99}}{0}
& \pmformat{74}{1}
& \pmformat{88}{5}
& \pmformat{4}{2}
& \pmformat{71}{8}
\\

\texttt{large-stitch} 
& \pmformat{26}{6} 
& \pmformat{62}{8} 
& \pmformat{51}{10} 
& \pmformat{28}{17} 
& \pmformat{40}{3} 
& \pmformat{\mathbf{97}}{3}
& \pmformat{40}{1}
& \pmformat{89}{15}
& \pmformat{0}{0}
& \pmformat{18}{4}
\\

\texttt{giant-stitch} 
& \pmformat{0}{0} 
& \pmformat{3}{6} 
& \pmformat{0}{0} 
& \pmformat{0}{0} 
& \pmformat{0}{0} 
& \pmformat{\mathbf{54}}{17}
& \pmformat{0}{0}
& \pmformat{\textbf{54}}{9}
& \pmformat{0}{0}
& \pmformat{0}{0}
\\

\texttt{teleport-stitch} 
& \pmformat{44}{7} 
& \pmformat{47}{2} 
& \pmformat{40}{7}
& \pmformat{48}{3} 
& \pmformat{45}{4} 
& \pmformat{\mathbf{66}}{3}
& \pmformat{27}{1}
& \pmformat{2}{4}
& \pmformat{10}{5}
& \pmformat{30}{3}
\\

\midrule
{average} 
& {31} 
& {50} 
& {46}
& {37} 
& {45} 
& \textbf{79}
& {35}
& {58}
& {4}
& {30}
\\

\bottomrule
\end{tabular}
\vspace{-1em}
\end{table}

\begin{figure}[b]
    \vspace{-1em}
    \centering
    \includegraphics[width=1.0\linewidth]{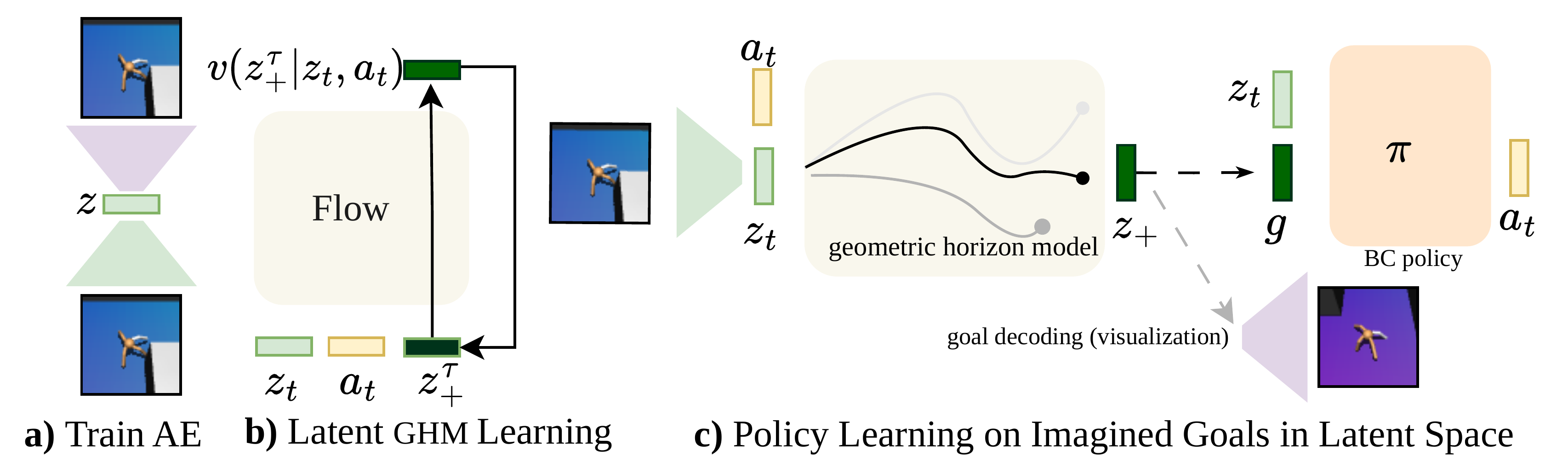}
    \caption{\textbf{Extension of \method to visual  environments.} }
    \label{fig:visual}
\end{figure}

Next, we study how geometric horizon models can \textbf{empirically} improve a policy's \textbf{combinatorial generalization} capabilities in Table \ref{tab:initial}. Without any augmentation, a goal-conditioned flow-based behavior cloning (GCFBC) baseline \citep{park2025flow} obtains moderate zero-shot generalization in smaller mazes, but struggles in larger mazes as the horizon is increased. To improve generalization, we apply our generative goal augmentation procedure, \methodbf. In addition to \method using $\mathcal{L}_{\text{MCHC}\text{-}\beta}$, we compare to \methodbase using $\mathcal{L}_{\text{TD}^2\text{-}\beta}$, as well as \methodmc using  $\mathcal{L}_\text{MC}$. We compare to clustering-based goal augmentation (Cluster-Aug) \citep{ghugare2024closinggaptdlearning}, which can create stitches between two trajectories, and against autoregressive 1-step prediction with history to sample from the SM (Dyn-Aug), as utilized in Figure \ref{fig:reload_emd_xy}. For these methods,  we list additional training details in Appendix \ref{sec:training_details}, including  hyperparameters for BC in Table \ref{t:bchparams}, and for TD-Flows in Table \ref{t:flowhparams}. Although the task horizon varies, we use the same hyperparameters for each environment in \texttt{pointmaze} for \method. This includes the effective horizon, which we set to $500$ ($\gamma=0.998$), while for \methodmc and \methodbase we use $100$ ($\gamma=0.99$), as they perform worse at higher horizons. After training GHMs, all policies are trained for $10^6$ steps, and each method is trained with batches using goals that are of a proportion $1- \alpha$ of future states in the same trajectory (i.e. hindsight goal sampling), and with proportion $\alpha$ with generated goals. The BC baseline uses $\alpha=0$, and augmentation uses $\alpha=0.5$. We can see that \method significantly improves generalization, including on long-horizon tasks, such as \texttt{pointmaze-giant-stitch}.  In contrast, \methodmc maintains close to original performance, demonstrating that generating additional synthetic short-horizon samples does not improve generalization. We also see that with \method, GCFBC outperforms prior offline GCRL methods in \texttt{pointmaze}, including GCIQL \citep{kostrikov2021offlinereinforcementlearningimplicit}, QRL \citep{wang2023optimalgoalreachingreinforcementlearning}, Contrastive RL (CRL) \citep{eysenbach2023contrastivelearninggoalconditionedreinforcement}, and HIQL \citep{park2023hiql}. Observing the weaker performance of CRL on \texttt{stitch} tasks, which also uses MC samples like GCBC, we show contrastive learning on generated samples leads to improved performance in Appendix \ref{sec:CRL}. 

\begin{table}[h]\centering
\tabcolsep=3pt
\cmidrulewidth=0.5pt
\fontsize{8}{8}\selectfont
\caption{\textbf{Visual generalization through compositional data augmentation.} With autoencoder (AE) GCFBC, we compare training on the original visual datasets to various data augmentation strategies. We calculate success rate averaged over the last 3 checkpoints, and then report the mean, std over 3 seeds. We \textbf{bold} the highest performing method overall, and highlight in \textcolor{brown}{\textbf{brown}} the top BC method.}
\vspace{-1em}
\label{tab:visual-initial}

\begin{tabular}{lllll|l|llll}\toprule
\text{\centering \texttt{visual-antmaze} } &
\text{\centering AE-GCFBC} &
\text{\shortstack{+\methodmc}} &
\text{\shortstack{+\methodbase}} &
\text{\shortstack{+\methodhc}} &
{GCBC} &
{GCIQL} &
{QRL} &
{CRL} &
{HIQL}
\\\midrule

\texttt{medium-stitch} &
\pmformat{45}{4} &
\pmformat{60}{2} &
\pmformat{56}{1} &
\pmformat{\textcolor{brown}{\textbf{84}}}{7} &
\pmformat{69}{4} &
\pmformat{9}{0} &
\pmformat{0}{0} &
\pmformat{82}{2} &
\pmformat{\mathbf{92}}{1}
\\

\texttt{large-stitch} &
\pmformat{12}{1} &
\pmformat{19}{1} &
\pmformat{9}{0} &
\pmformat{\textcolor{brown}{\textbf{28}}}{1} &
\pmformat{21}{6} &
\pmformat{1}{0} &
\pmformat{1}{0} &
\pmformat{29}{6} &
\pmformat{\mathbf{46}}{2}
\\

\midrule
\text{average} &
{29} &
{40} &
{33} &
{\color{brown}{\textbf{56}}} &
45 &
5 &
1 &
56 &
\textbf{69}
\\

\bottomrule
\end{tabular}
\vspace{-1em}
\end{table}

\begin{figure}[h]
    \vspace{-1em}
    \centering
    \includegraphics[width=1.0\linewidth]{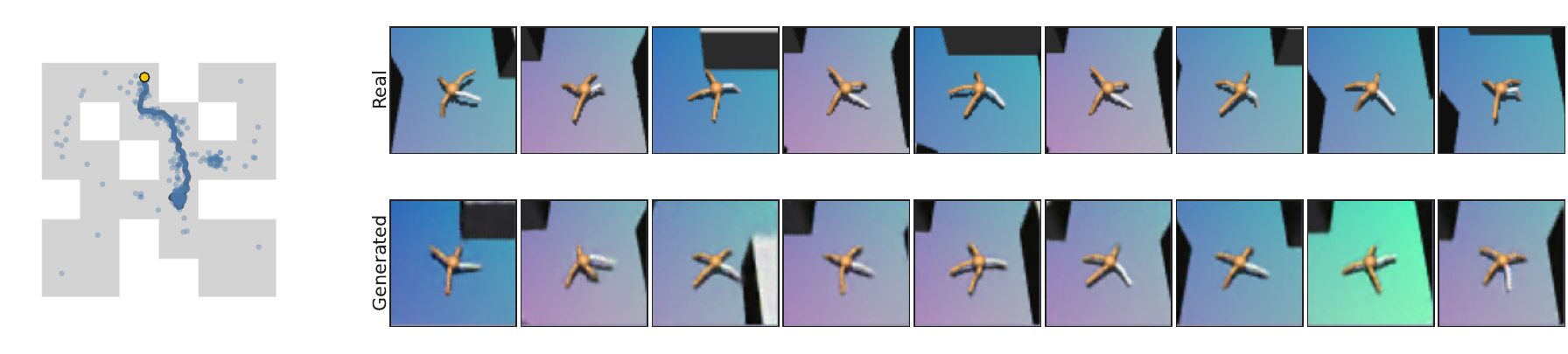}
    \caption{\textbf{Samples for \texttt{visual-antmaze-medium}.} We visualize samples generated with $\text{MCHC-}\beta$ at $\gamma=0.995$. We display both real, geometrically sampled, states, as well as reconstructions from the GHM's predicted latents using the AE decoder. }
    \label{fig:placeholder}
    \vspace{-1em}
\end{figure}

\subsection{Composition in Visual Navigation}
\label{sec:visual}
In this section, we extend \methodbf to \textit{visual environments} with \texttt{visual-antmaze-stitch}. Instead of learning the GHM in pixel space, we first train an autoencoder (AE), obtaining a low-dimensional latent space $\mathcal{Z}$. We provide details for training the AE in Table \ref{t:aehparams}. Then, we utilize the same architectures for the GHM, and policy, as performed in the non-visual setting (Section \ref{exp:state}), but now operate over latent space, as shown in Figure \ref{fig:visual}. As a baseline, we report both standard BC in OGBench (\textbf{GCBC}), using early state-goal fusion, and an autoencoder BC baseline, labeled \textbf{AE-GCFBC}. The standard baseline concatenates state and goal images, and passes them through a small IMPALA-based \citep{espeholt2018impala} CNN. In contrast, the autoencoder BC baseline separately encodes goals and states into $\mathcal{Z}$, and then concatenates them to predict actions. This architecture is preferable with goal augmentation, as we generate goals in latent space \citep{dao2023lfm} and feed them to the policy without decoding. With this setup, in Table \ref{tab:visual-initial}, we report adding augmentation to AE-GCFBC. We can see that without augmentation, this architecture under-performs GCBC. While we improve with \methodmc, we can see that $\mathcal{L}_{\text{TD}^2\text{-}\beta}$, used in \methodbase leads to less performance, which is in contrast to the non-visual setting. We attribute this to the difficulty of bootstrapping in higher-dimensional  approximate ($\mathcal{Z}$) space. Thus, we find the combination of MC and TD to be the most effective, where \method improves AE-GCFBC and outperforms the GCBC baseline. 

\begin{table}[h]\centering
\caption{\textbf{Manipulation with compositional data augmentation.} We evaluate GCFBC with action chunking on OGBench manipulation tasks, using the \texttt{oracle} goal representation. We use the original \texttt{play} dataset, and then introduce a new \texttt{stitch} dataset, consisting of only short trajectories. Then, we add \method to close the generalization gap from \texttt{stitch}. We calculate average success on the last 3 checkpoints, and report the mean, std over 3 seeds. We \textcolor{brown}{\textbf{highlight}} the top method on \texttt{stitch}. }
\vspace{-1em}

\label{tab:manip}
\centering
\tabcolsep=3pt
\cmidrulewidth=0.5pt
\fontsize{8}{8}\selectfont
\begin{tabular}{lc|ccc|cc}
\toprule

\multirow{3}{*}{\texttt{dataset}}
& \multicolumn{2}{c}{GCFBC}
& \multicolumn{1}{c}{+\methodbase}
& \multicolumn{1}{c}{+\methodhc} 
& \multicolumn{1}{c}{GCIQL}
& \multicolumn{1}{c}{HIQL} \\

\cmidrule(lr){2-7}

& \texttt{play}
& \texttt{stitch}
& \texttt{stitch}
& \texttt{stitch}
& \texttt{stitch}
& \texttt{stitch} \\
\midrule

\texttt{cube-double}
& \pmformat{75}{1}
& \pmformat{20}{0}
& \pmformat{78}{2}
& \pmformat{\textcolor{brown}{\textbf{84}}}{0}
& \pmformat{63}{1} 
& \pmformat{6}{2}
\\

\texttt{cube-triple}
& \pmformat{35}{5}
& \pmformat{20}{0}
& \pmformat{\textcolor{brown}{\textbf{36}}}{5}
& \pmformat{31}{2}
& \pmformat{8}{2} 
& \pmformat{0}{0} 

\\

\texttt{puzzle-4x4}
& \pmformat{23}{2}
& \pmformat{0}{0}
& \pmformat{18}{0}
& \pmformat{\textcolor{brown}{\textbf{40}}}{2}
& \pmformat{16}{5}
& \pmformat{0}{0}
\\

\midrule

\text{average}
& 44
& 13
& 44
& \textcolor{brown}{\textbf{52}}
& 29
& 2
\\

\bottomrule
\end{tabular}
\vspace{-2em}
\end{table}

\subsection{Composition in Manipulation}
\label{sec:manip}
In this section, we evaluate TD-Aug on OGBench manipulation tasks requiring stitching, which is a challenging setting for modeling the successor measure. We first consider \texttt{cube-\{double,triple\}-play}, \texttt{puzzle-4x4-play} tasks with \texttt{oracle} goal representation \citep{park2024ogbenchbenchmarkingofflinegoalconditioned, park2026horizon}. In addition to the original  \texttt{play} datasets, we construct a new challenging \texttt{stitch} dataset, where each trajectory contains one atomic task (i.e. pick-and-place, or button press), but the policy is evaluated on the original, long-horizon, multi-step tasks. While GCFBC using action chunking, $\pi(a_{t:t+h}|s,g)$, obtains good average success $(44\%)$ on the \texttt{play} tasks, it struggles when training on the \texttt{stitch} $(13\%)$. Thus, we apply TD-Aug to close the gap. We also train GHMs, with action chunking, $\smash{M^\beta_\phi(\cdot | s, a_{t:t+h})}$, and predict $h\text{-sized}$ jumps in the full state space, with additional details in Appendix \ref{sec:GHMmanipalg}. On \texttt{stitch}, with TD-Aug, GCFBC's performance substantially increases from $13 \rightarrow 52$. We provide visualizations of the GHM's samples in Figure \ref{fig:manipsamples}. 
\vspace{-1em}
\section{Conclusion}
In this work, we have shown that TD-based geometric horizon models allow for enhanced generalization for behavior cloning through goal augmentation. There are a variety of directions that could build on this work. For example, behavior cloning is a strong pre-training algorithm that has been shown to scale well in real-world robotics settings \citep{o2024open, black2026pi0visionlanguageactionflowmodel, intelligence2025pi06vlalearnsexperience}. Therefore, a key open question is how to efficiently and stably scale TD-based GHM training to high-dimensional, partially-observable robotics tasks.  %

\subsection*{Acknowledgments}
We want to acknowledge funding support from Natural Sciences and Engineering Research Council (NSERC) of Canada, Google Research, Fonds de recherche du Québec (FRQNT), and the Canadian Institute for Advanced Research (CIFAR) and compute support from Digital Research Alliance of Canada, Mila IDT, and NVidia. We also thank Jesse Farebrother and Pablo Samuel Castro for advice throughout the project.

\bibliography{references}
\bibliographystyle{rlj}

\newpage

\appendix

\section{Training Details}
\label{sec:training_details}
\begin{table}[bh]
  \centering
  \begin{tabularx}{\linewidth}{@{}p{0.24\linewidth} Y l@{}}
    \toprule
    & \textbf{Hyperparameter} &  \\
    \midrule

    \multirow{3}{=}{Flow Matching}
      & ODE Solver & Midpoint \\
      & Train ODE $\mathrm{d}\tau$ & $0.1$ \\
      & Sample ODE $\mathrm{d}\tau$  & $0.1$ \\
    \midrule

    \multirow{5}{=}{Network (U-Net)}
      & $t$-Positional Embedding Dim & $256$ \\
      & $t$-Positional Embedding MLP & $(256,256)$ \\
      & Block Dimension & $(512,512,512)$ \\
      & Encoder MLP & $(512,512,512)$ \\
      & activation & Mish \\
    \midrule

    \multirow{5}{=}{Optimizer (AdamW)}
      & AdamW $\beta_1$ & $0.9$ \\
      & AdamW $\beta_2$ & $0.999$ \\
      & AdamW $\epsilon$ & $10^{-4}$ \\
      & Learning Rate & $10^{-4}$ \\
      & Weight Decay & $10^{-3}$ \\
    \midrule 

    \multirow{7}{=}{Training}
      & Target Network EMA & $10^{-3}$ \\
      \cmidrule(l){2-3}
      & Gradient Steps & \\
      & \quad TD &  \\
      & \quad \quad State,Visual & 3M \\
      & \quad \quad Manipulation & 1M \\
      & \quad MC, MCHC & 1M \\
      \cmidrule(l){2-3}
      & Batch Size & 1024 \\
    \cmidrule(l){2-3}
      & Discount ($\gamma$) &  \\
        & \quad MC, TD & $0.99$ \\
        & \quad MCHC &  \\
        & \quad \quad State & $0.998$ \\
        & \quad \quad Visual & $0.995$ \\
        & \quad \quad Manipulation & $0.99$ \\

    & Discount ($\omega$) & \\
      & \quad \quad State,Visual & $0.99$ \\
      & \quad \quad Manipulation & $0.98$ \\

        \bottomrule
  \end{tabularx}

  \caption{Hyperparameters for TD-Flow training, following similar architecture details to \cite{farebrother2025temporal}. }
  \label{t:flowhparams}
\end{table}

\begin{table}[h]
  \centering
    \caption{Hyperparameters for GCFBC training.}
  \label{t:bchparams}

  \begin{tabularx}{\linewidth}{@{}p{0.24\linewidth} Y l@{}}
    \toprule
    & \textbf{Hyperparameter} &  \\
    \midrule

    \multirow{2}{=}{Flow Matching}
      & ODE Solver & Euler \\
      & Sample ODE $\mathrm{d}\tau$  & $0.1$ \\
    \midrule

    \multirow{2}{=}{Network (MLP)}
      & Hidden Layers & $(1024,1024,1024,1024)$ \\
      & activation & GELU \\
    \midrule

    \multirow{5}{=}{Optimizer (Adam)}
      & Adam $\beta_1$ & $0.9$ \\
      & Adam $\beta_2$ & $0.999$ \\
      & Learning Rate & \\
      & \quad State, Visual  & $3 \times 10^{-4}$ \\
      & \quad Manipulation & $3 \times 10^{-5}$ \\
      & Weight Decay & $0$ \\
    \midrule 

    \multirow{6}{=}{Training}
      & Gradient Steps &  \\
      & \quad State & 1M \\
      & \quad Manipulation & 2M \\

      & \quad Visual & 500k \\

      \cmidrule(l){2-3}
      & Batch Size & \\
      & \quad State & 1024 \\
      & \quad Manipulation & 1024 \\
      & \quad Visual & 256 \\

      \midrule
    \multirow{1}{=}{Augmentation (only) }
      & Batch Proportion ($\alpha$) & $0.5$ \\

    \bottomrule
  \end{tabularx}

\end{table}

\begin{table}[h]
  \centering
    \caption{Hyperparameters for visual autoencoder training. For our autoencoder, we use a similar architecture to \citep{kostrikov2020image,yarats2021improving}. For the encoder, we utilize two convolutional layers, with ReLU activations after each layer. The output of the convolutional layers is flattened, and linearly mapped to the latent dimension $|z| = 50$, followed by LayerNorm and tanh activation. The decoder utilizes a similar architecture, utilizing deconvolution layers, with the same kernel size, followed by ReLU (except for the final layer). }
  \label{t:aehparams}

  \begin{tabularx}{\linewidth}{@{}p{0.24\linewidth} Y l@{}}
    \toprule

    & \textbf{Hyperparameter} &  \\
    \midrule

    \multirow{5}{=}{Architecture}
      & Convolution Layers &  $2$ \\
      & Channels  & $32$ \\
      & Kernel & $(3 \times 3)$ \\
      & Activation & ReLU \\
      & Latent Dimension, $|z|$ & $50$ \\
      \midrule

    \multirow{4}{=}{Optimizer (Adam)}
      & Adam $\beta_1$ & $0.9$ \\
      & Adam $\beta_2$ & $0.999$ \\
      & Learning Rate & $ 10^{-4}$ \\
      & Weight Decay & $0$ \\
      \midrule
      
    \multirow{1}{=}{Training}
      & Gradient Steps & 100k  \\
      & Batch Size & 1024 \\

    \bottomrule
  \end{tabularx}

\end{table}

\begin{table}[h]
\centering
\caption{\textbf{Hyperparameters for clustering-based goal-augmentation.} We follow \citep{ghugare2024closinggaptdlearning} and use k-means clustering \citep{lloyd1982least}, with $k$ set as below. }
\label{tab:hparams}

\tabcolsep=6pt
\renewcommand{\arraystretch}{1.1}
\begin{tabular}{lr}
\toprule
\textbf{Hyperparameter} &  \\
\midrule
 Clusters  &  \\
\quad \texttt{medium} & 40 \\
\quad \texttt{large} & 40 \\
\quad \texttt{teleport} & 40 \\
\quad \texttt{giant} & 80 \\

\bottomrule
\end{tabular}
\end{table}

\newpage
\section{Additional Objectives}
\label{sec:objs}
\subsection{TD-HC}
\label{sec:TD-HC}
For reference, we provide the original  Temporal Difference Horizon Consistency (TDHC) objective \citep{farebrother2026compositionalplanningjumpyworld}, which we modify to use MC samples in Section \ref{sec:MCHC}:
\begin{align}
\label{eq:td2-hc}
&\mathcal{L}_{\mathrm{TDHC}}(\phi ; \bar\phi) = (1 - \textcolor{blue}{\gamma}) \mathcal{L}_{\textrm{td-ss}}(\phi; \textcolor{blue}{\gamma}) +  \textcolor{blue}{\gamma}\frac{1 - \textcolor{blue}{\gamma}}{1 - \textcolor{orange}{\omega}}\mathcal{L}_{\textrm{td-bs}}(\phi; \bar\phi, \textcolor{blue}{\gamma}, \textcolor{orange}{\omega} ) +\textcolor{blue}{\gamma}\frac{\textcolor{blue}{\gamma}- \textcolor{orange}{\omega}}{1 - \textcolor{orange}{\omega}}\mathcal{L}_{\textrm{td-hc-bs}}(\phi; \bar\phi, \textcolor{blue}{\gamma}, \textcolor{orange}{\omega})\\
&\mathcal{L}_{\textrm{td-ss}}=\mathbb{E}_{\substack{\tau\sim U(0,1), (s,a,s')\sim \mathcal{D}^\beta \\ x_0 \sim p_0, x_\tau=(1-\tau)x_0 + \tau s'}}\left[\|v^\tau_\phi(x_\tau |s,a, \textcolor{blue}{\gamma}) - (s' - x_0)\|^2\right]\notag\\
&\mathcal{L}_{\textrm{td-bs}}=\mathbb{E}_{\substack{\tau\sim U(0,1), (s,a,s')\sim \mathcal{D}^\beta, a' \sim \pi(s') \\ x_0 \sim p_0, x_\tau=\psi^\tau_{\bar \phi}(x_0 | s',a', \textcolor{orange}{\omega})}}\left[\|v^\tau_\phi(x_\tau |s,a, \textcolor{blue}{\gamma}) - v^\tau_{\bar\phi}(x_\tau |s',a', \textcolor{orange}{\omega})\|^2\right]\notag\\
&\mathcal{L}_{\textrm{td-hc-bs}}=\mathbb{E}_{\substack{\tau\sim U(0,1), (s,a,s') \sim \mathcal{D}^\beta, a' \sim \pi(s') \\ x_0^1, x_0^2 \sim p_0,s_+ = \psi^1_{\bar\phi}(x_0^1 | s',a', \textcolor{orange}{\omega}), a_+ \sim \pi(s_+)
\notag \\ x_\tau = \psi^\tau_{\bar \phi}(x_0^2 | s_+,a_+, \textcolor{blue}{\gamma})}}\left[\|v^\tau_\phi(x_\tau |s,a, \textcolor{blue}{\gamma}) - v^\tau_{\bar\phi}(x_\tau |s_+,a_+, \textcolor{blue}{\gamma})\|^2\right]\
\end{align}
With $\textcolor{blue}{\gamma} =\textcolor{orange}{\omega}$, the original $\mathrm{TD}^2$ objective \eqref{eq:td2} is recovered. To train with the TDHC objective, the flow is trained with discount conditioning, and in each batch, the discount is sampled uniformly $\textcolor{blue}{\gamma} \in [\gamma_{\mathrm{min}},\gamma_{\mathrm{max}}]$ and $\textcolor{orange}{\omega} \in [\gamma_{\mathrm{min}}, \textcolor{blue}{\gamma}]$. Instead, the  MCHC objective \eqref{eq:td2-mc-hc} we use fixes $\textcolor{blue}{\gamma} = \gamma_{\mathrm{max}}$, and samples $(s_+, a_+) \sim D^\beta$, where $s_+ = s_{t+k}$, $k \sim \mathrm{geom}(1-\textcolor{orange}{\omega})$, replacing having to learn and sample with $v_\phi^\tau(\cdot | \cdot, {\textcolor{orange}{\omega}})$.   

\subsection{GHMs for Manipulation}
\label{sec:GHMmanipalg}

For manipulation tasks, we modify both GCFBC and the GHM to utilize action chunks of size $H$ as shown in \eqref{eq:td2-mc-hc-chunk}. For the GHM we also predict future states starting from after the action chunk, at increments matching the chunk size. We also proportionally reduce $\gamma, \omega$. For example, with $h=10$, we reduce $\gamma$ from $0.99$ to $0.9$, to match the same expected horizon. Additionally, we bootstrap with samples from a learned prior, $\beta_\theta(\cdot | s_+)$, predicting the next action chunk, rather than using $a_+$. This is useful, as it can allow for bootstrapping from the end states in a trajectory, where an action chunk is not available. 

\begin{align}
\label{eq:td2-mc-hc-chunk}
&\mathcal{L}_{\text{MCHC-}\beta}(\phi ; \bar\phi) = (1 - \textcolor{blue}{\gamma}) \mathcal{L}_{\textrm{td-ss}}(\phi) +  \textcolor{blue}{\gamma}\frac{1 - \textcolor{blue}{\gamma}}{1 - \textcolor{orange}{\omega}}\mathcal{L}_{\textrm{mc}}(\phi) +\textcolor{blue}{\gamma}\frac{\textcolor{blue}{\gamma}- \textcolor{orange}{\omega}}{1 - \textcolor{orange}{\omega}}\mathcal{L}_{\textrm{td-mc-bs}}(\phi ; \bar\phi)\\
&\mathcal{L}_{\textrm{td-ss}}(\phi)=\mathbb{E}_{\substack{\tau\sim U(0,1), (s_t,a_{t:t+h},s_{t+h})\sim \mathcal{D}^\beta \\ x_0 \sim p_0, x_\tau=(1-\tau)x_0 + \tau s_{t+h}}}\left[\|v^\tau_\phi(x_\tau |s,a_{t:t+h}) - (s_{t+h} - x_0)\|^2\right]\notag\\
&\mathcal{L}_{\textrm{mc}}(\phi)=\mathbb{E}_{\substack{\tau\sim U(0,1), (s_t,a_{t:t+h},s_+)\sim \mathcal{D}^\beta\\ x_0 \sim p_0, x_\tau =(1-\tau)x_0 + \tau s_{+}}}\left[\|v^\tau_\phi(x_\tau |s,a_{t:t+h}) - (s_+ - x_0)\|^2\right]\notag\\
&\mathcal{L}_{\textrm{td-mc-bs}}(\phi)=\mathbb{E}_{\substack{\tau\sim U(0,1), (s_t,a_{t:t+h},s_+)\sim \mathcal{D}^\beta \\ x_0 \sim p_0, a_+ \sim \beta_\theta(\cdot | s_+), x_\tau = \psi^\tau_{\bar \phi}(x_0 | s_+,a_+)}}\left[\|v^\tau_\phi(x_\tau |s_t,a_{t:t+h}) - v^\tau_{\bar\phi}(x_\tau |s_+,a_+)\|^2\right]\ \notag\\
&\text{where: } s_+ = s_{t + h(1 + k)}, k \sim \mathrm{geom}(1-\textcolor{blue}{\omega}), \text{ and } a_+ \text{ also uses chunk size } h = 10.\notag
\end{align}

\subsection{One-step Dynamics Model}
\label{sec:one-step}
As a baseline to geometric horizon models, we utilize auto-regressive prediction with one-step models to sample from the successor measure. This follows from first sampling a random offset $k \sim \mathrm{geom}(1-\gamma)$.  We then perform autoregressive prediction for $k$ steps, first sampling $s_{t+1} \sim p(\cdot | s,a)$, and then sampling $k-1$ times from $s_{t+i+1} \sim p^\beta(\cdot  |s_{t+i})$, for $i=1$ to $k -1$. We parameterize $p(\cdot | s,a)$ and $p^\beta(\cdot  |s)$ via velocity field $v^\tau_\phi(x_\tau | s,a)$, trained via flow matching:
\begin{align}
\mathcal{L}_{\textrm{1-step}}(\phi)
&=
\mathbb{E}_{\substack{
\tau \sim \mathcal{U}(0,1),\,
(s,a,s') \sim \mathcal{D}^{\beta},\,
a_m \sim \mathrm{mask}(a) \\
x_0 \sim p_0,\,
x_\tau = (1-\tau)x_0 + \tau s'
}}
\left[
\left\|
v^\tau_\phi(x_\tau \mid s,a_m) - (s' - x_0)
\right\|_2^2
\right]
\end{align}
We utilize a standard 1-step flow matching objective, however we condition on $a_m \sim \mathrm{mask}(a)$, where  with probability $0.5$, the action is replaced by a learned embedding (shared for all actions). Thus, this objective trains both forms of $p$. We utilize the same architecture as GHM as described in Table \ref{t:aehparams}, and train for 1 million steps. 

While in theory, this objective can perfectly capture the successor measure, we found issues with compounding error during autoregressive prediction as shown in Figure \ref{fig:reload_emd_xy}. This is particularly due to $p^\beta(\cdot  |s)$ being memory-less, i.e. after the first action, there is no conditioning information in $p^\beta(\cdot  |s)$ that captures an intention (such as running to the right). Thus, we find a stronger performing baseline incorporates history, leading to:
\begin{align}
\mathcal{L}_{\text{history 1-step}}(\phi)
&=
\mathbb{E}_{\substack{
\tau \sim \mathcal{U}(0,1),\,
(s_{t-h+1: t},a,s_{t+1}) \sim \mathcal{D}^{\beta}\\
x_0 \sim p_0,\,
x_\tau = (1-\tau)x_0 + \tau s_{t+1} \\
a_m \sim \mathrm{mask}(a) 
}}
\left[
\left\|
v^\tau_\phi(x_\tau \mid s_{t-h+1: t},a_m) - (s_{t+1} - x_0)
\right\|_2^2
\right]
\end{align}
Where we utilize a history length of $h=4$, as used in Figures \ref{fig:reload_emd_xy}, \ref{fig:largediscountsweep}. When sampling with a specific discount $\gamma$, with expected horizon $H=\frac{1}{1-\gamma}$, we cap $k$ for rollouts at $2H$. For example, if $\gamma=0.99$, we cap $k = \min(\mathrm{geom}(1-0.99), 200)$. 

As GHMs can sample in O(1) steps, when training a policy, we can efficiently sample a fresh batch of goals, $s_+^i \sim M^\beta_\gamma(\cdot  | s_t,a_t)$, for all batch indices, every time we form a batch for the policy update. However, this is not realistic with one-step models at $\gamma=0.99$. Thus, we instead cache a $2H =200$ sized trajectory from each state from the dataset, storing $(N, 200)$ states, where $N$ is the original number of transitions (e.g. $1,000,000$). Then, while training the policy, we geometrically sample future states from this window. We utilize this procedure when performing augmentation with one-step dynamics models (Dyn-Aug) as in Table \ref{tab:initial}.

\newpage

\section{Relation to 1-step policy improvement}
\label{theorem:one-step-proof}

\begin{proof}
\renewcommand{\qedsymbol}{}
\begin{align}
&
\mathbb{E}_{\substack{
s \sim d^{\beta}(s) \\
a \sim \beta(a \mid s) \\
s_+ \sim d^{\beta}(s_+)
}}
\left[
\frac{Q^\beta(s,a, s_+)}{d^\beta(s_+)}
\log \pi_\theta(a \mid s, g=s_+)
\right] \\
&=
\frac{1}{1- \gamma}\mathbb{E}_{\substack{
s \sim d^{\beta}(s) \\
a \sim \beta(a \mid s) \\
s_+ \sim d^{\beta}(s_+)
}}
\left[
\frac{m^\beta(s_+ \mid s,a)}{d^\beta(s_+) }
\log \pi_\theta(a \mid s, g=s_+)
\right] \\
&=
\frac{1}{1- \gamma}\mathbb{E}_{\substack{
s \sim d^{\beta}(s) \\
a \sim \beta(a \mid s) \\
s_+ \sim M^{\beta}(s_+ \mid s,a)
}}
\left[
\log \pi_\theta(a \mid s, g=s_+)
\right].
\end{align}
The second equality uses the identity
$Q^\beta(s,a, s_+) = ({1-\gamma})^{-1}m^\beta(s_+ \mid s,a)$
in the goal-conditioned setting, and the final change corresponds to importance sampling.
\end{proof}

\newpage

\section{Samples}
\label{sec:samples}
\begin{figure}[h!]
    \centering
    \includegraphics[width=1.0\linewidth]{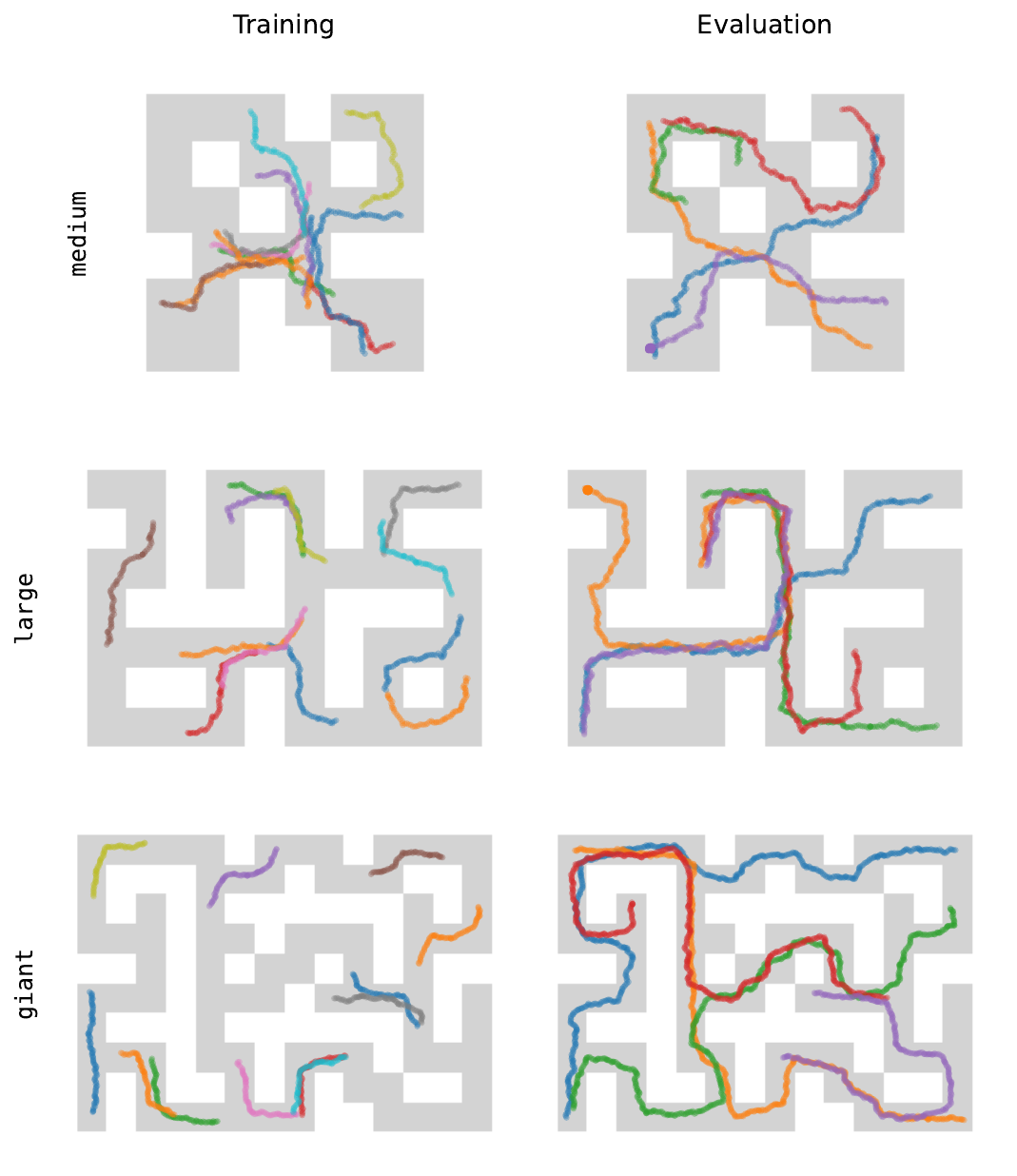}
    \caption{Example ground-truth training and evaluation trajectories on \texttt{pointmaze}.}
    \label{fig:largetraineval}
\end{figure}
\newpage

\begin{figure}[h!]
    \centering
    \begin{subfigure}{1.0\linewidth}    
     \includegraphics[width=1.0\linewidth]{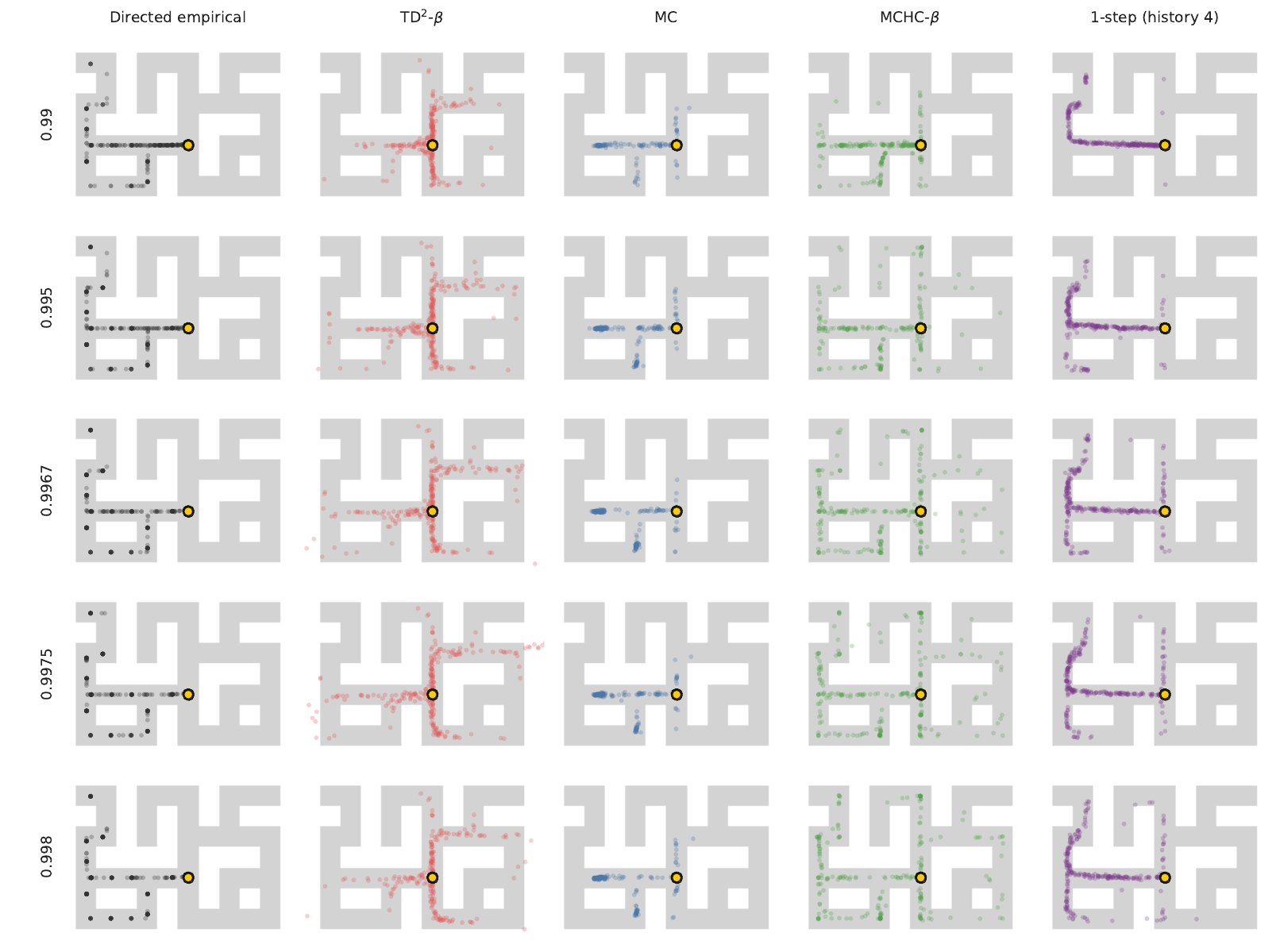}
    \caption{}
    \end{subfigure}
    
    \begin{subfigure}{1.0\linewidth}   
    \includegraphics[width=1.0\linewidth]{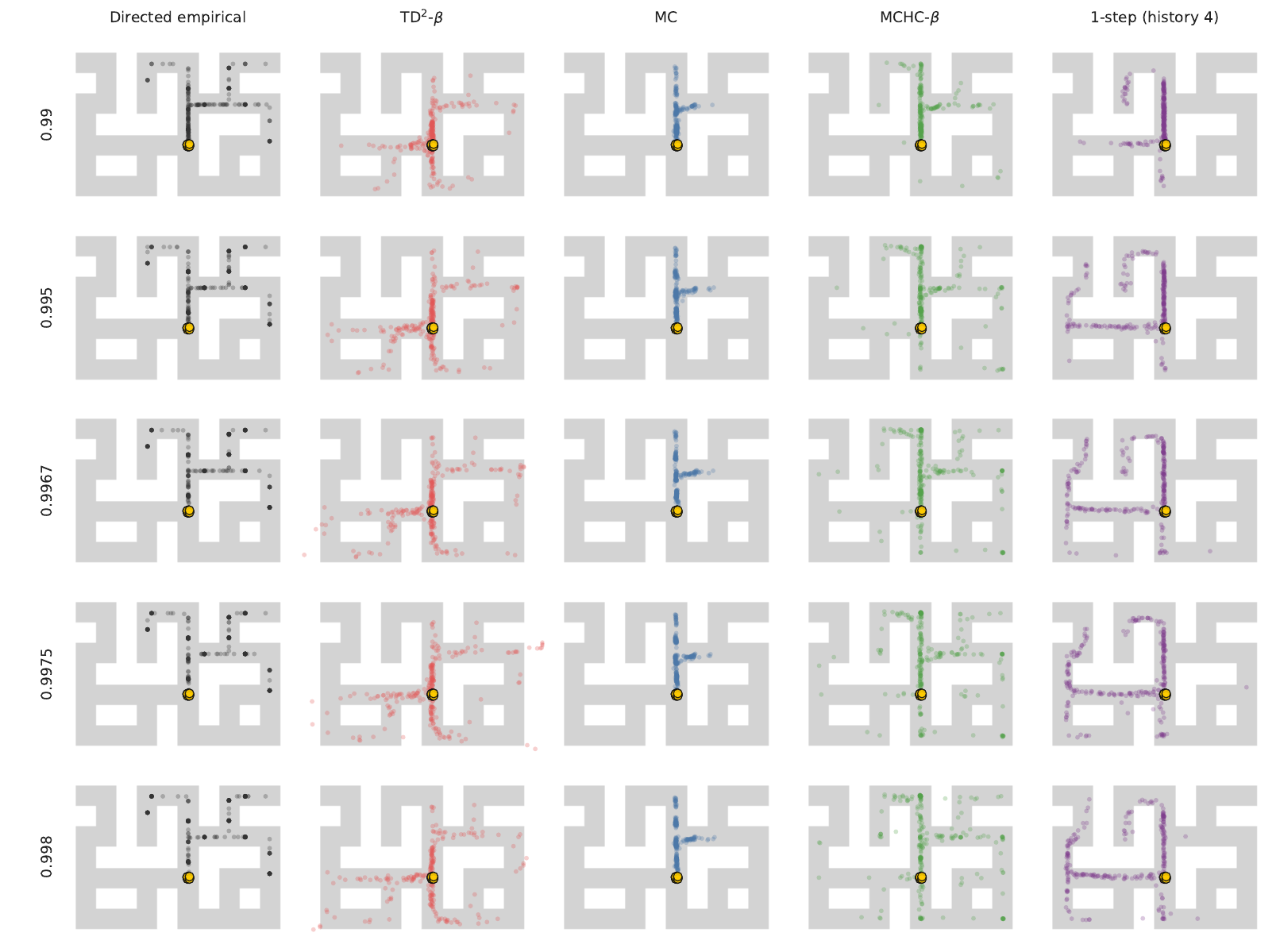}

    \caption{}
    \end{subfigure}

    \caption{\textbf{Generated samples in \texttt{pointmaze-large with varying horizon}.}
    Samples are conditioned on different sources displayed with \yellowcircle, and ground-truth action, whose effect can be inferred by the ground-truth trajectory shown in \textbf{black}. Across each row, we vary discount $\gamma$, with $H = (1- \gamma)^{-1}$. }
    \label{fig:largediscountsweep}
\end{figure}

\begin{figure}[h]
    \centering

    \begin{subfigure}{0.75\linewidth}
        \centering
        \includegraphics[width=\linewidth]{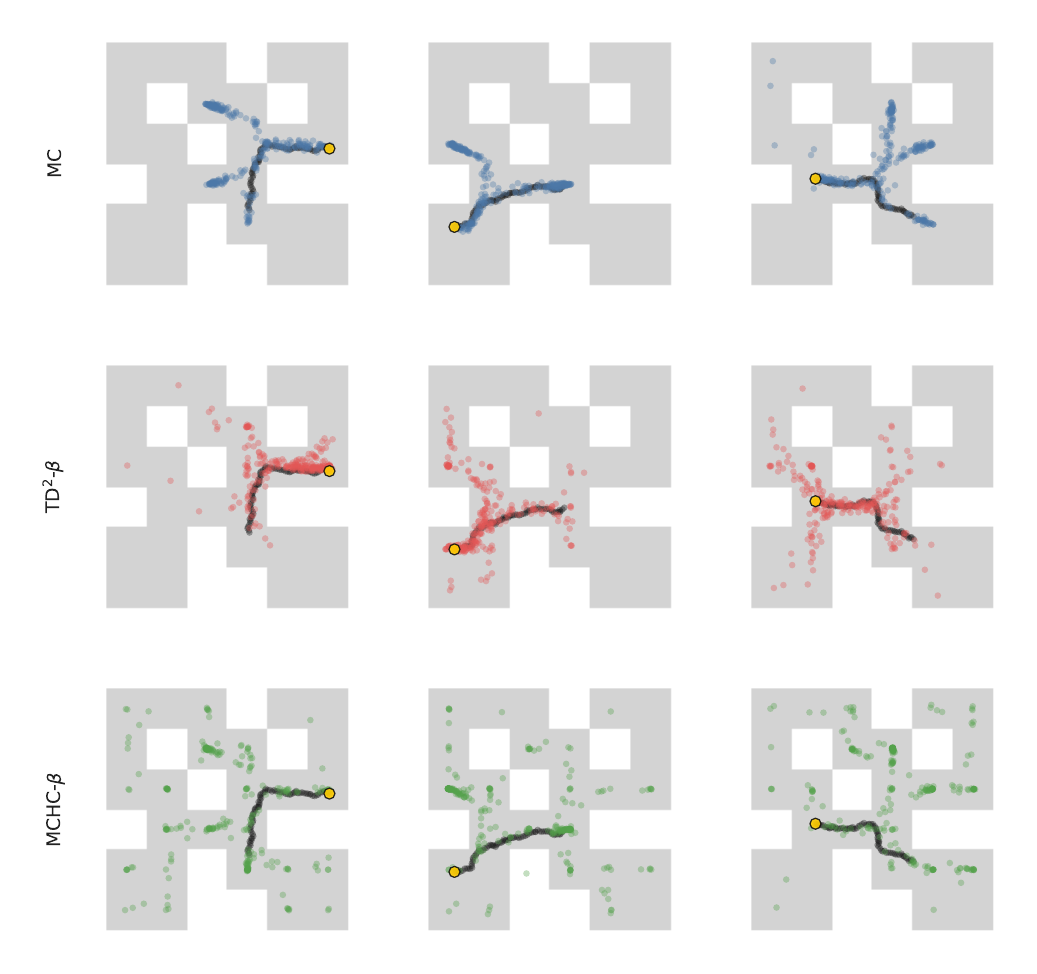}
        \caption{\texttt{pointmaze-medium}}
        \label{fig:pointmaze-medium-reload-sample-grid}
    \end{subfigure}

    \vspace{0.75em}

    \begin{subfigure}{0.9\linewidth}
        \centering
        \includegraphics[width=\linewidth]{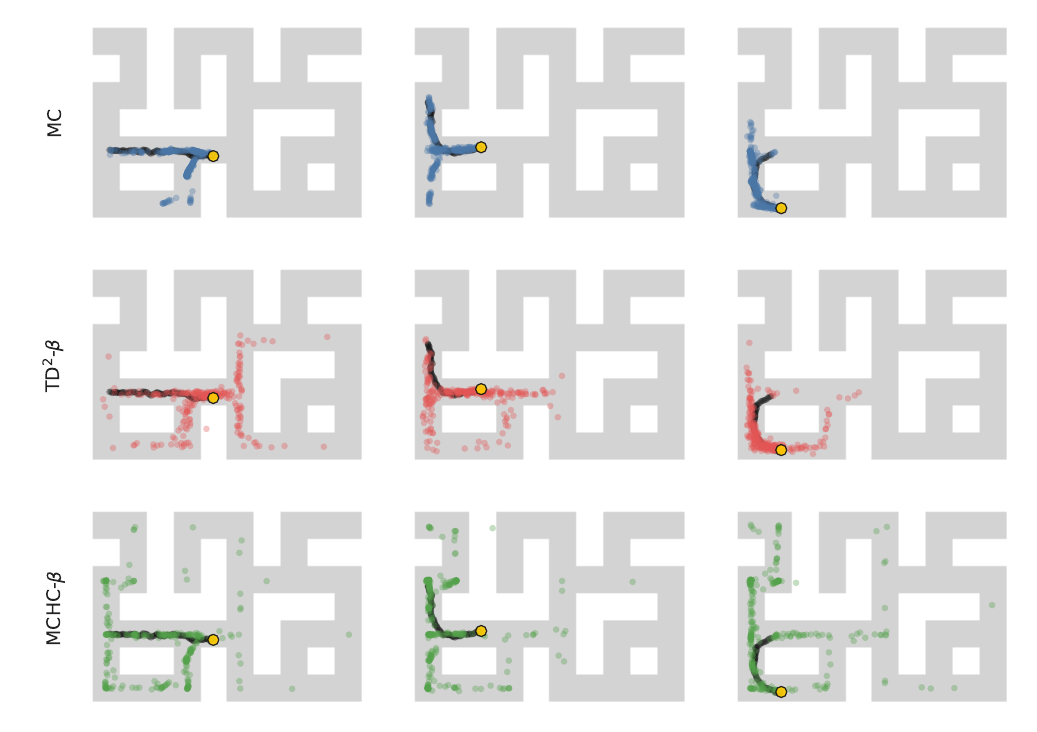}
        \caption{\texttt{pointmaze-large}}
        \label{fig:pointmaze-large-reload-sample-grid}
    \end{subfigure}

    \caption{\textbf{Generated samples in \texttt{pointmaze} environments.}
    Samples are conditioned on the source displayed with \yellowcircle, and ground-truth action, whose effect can be inferred by the ground-truth trajectory shown in \textbf{black}. These samples reflect the settings used in Table \ref{tab:initial}, where $\text{MCHC-}\beta$ uses $\gamma=0.998$ while $\text{TD}^2\text{-}\beta$ and MC use $\gamma =0.99$.  }
    \label{fig:pointmaze-sample-grid}
\end{figure}

\begin{figure}[p]
    \ContinuedFloat
    \centering

    \begin{subfigure}{0.9\linewidth}
        \centering
        \includegraphics[width=\linewidth]{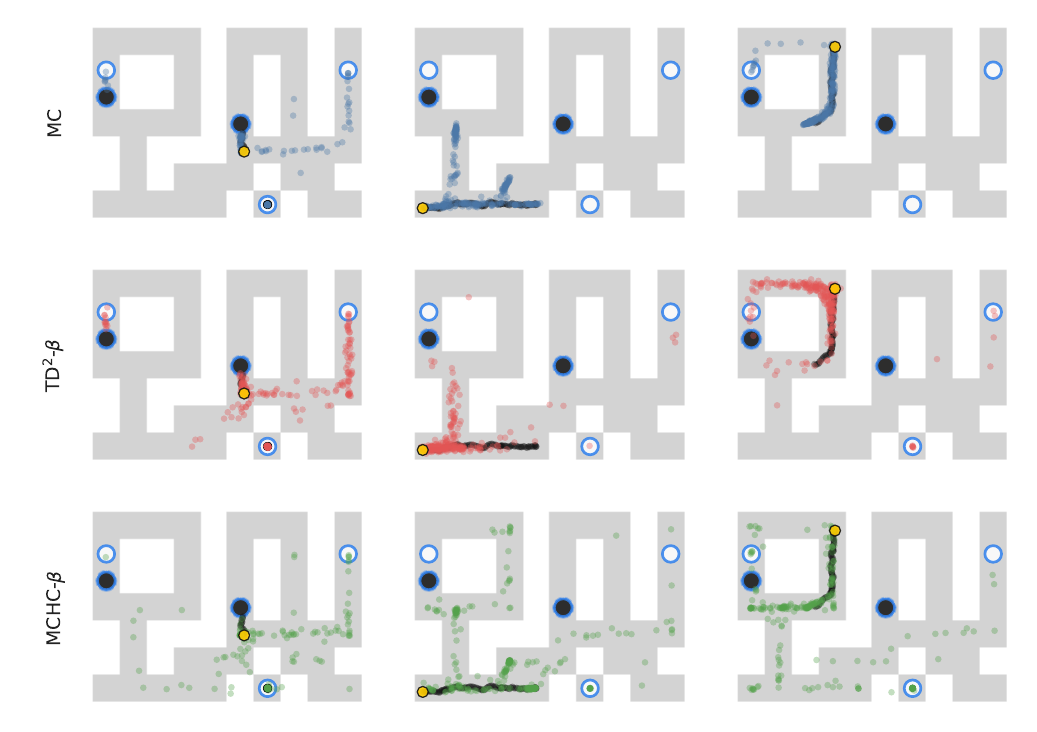}
        \caption{\texttt{pointmaze-teleport}}
        \label{fig:pointmaze-teleport-reload-sample-grid}
    \end{subfigure}

    \begin{subfigure}{0.95\linewidth}
        \centering
        \includegraphics[width=\linewidth]{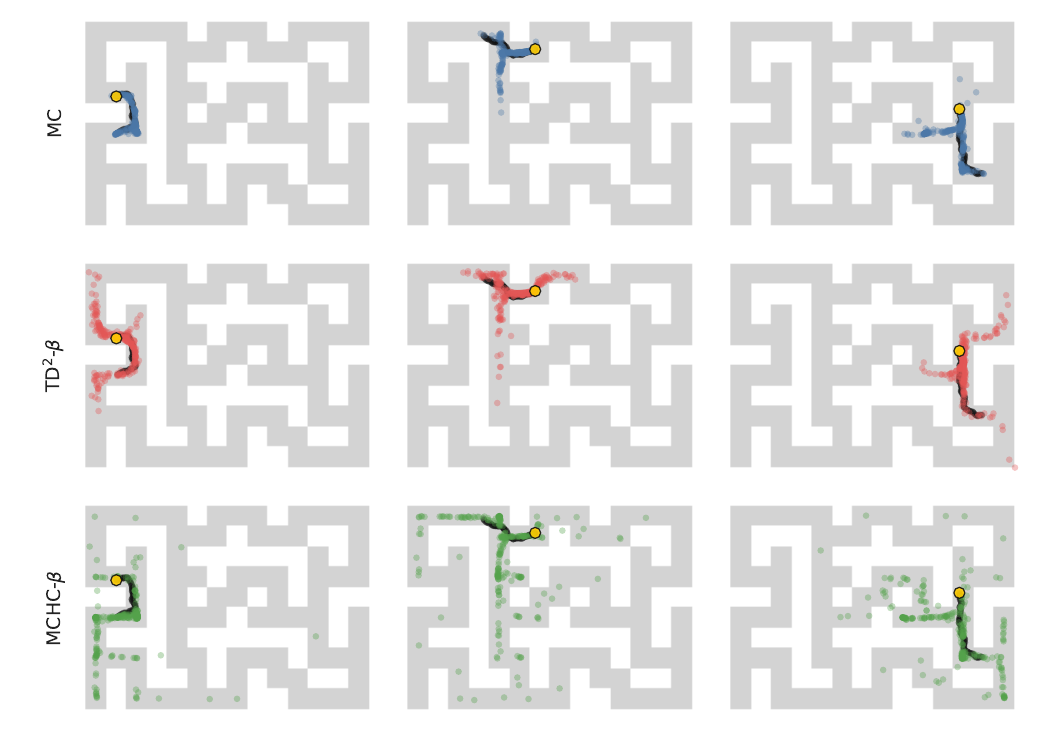}
        \caption{\texttt{pointmaze-giant}}
        \label{fig:pointmaze-giant-reload-sample-grid}
    \end{subfigure}

    \caption[]{\textbf{Generated samples in \texttt{pointmaze} environments, continued.}}
\end{figure}

\newpage
\begin{figure}[h!]
    \centering
    \includegraphics[width=.75\linewidth]{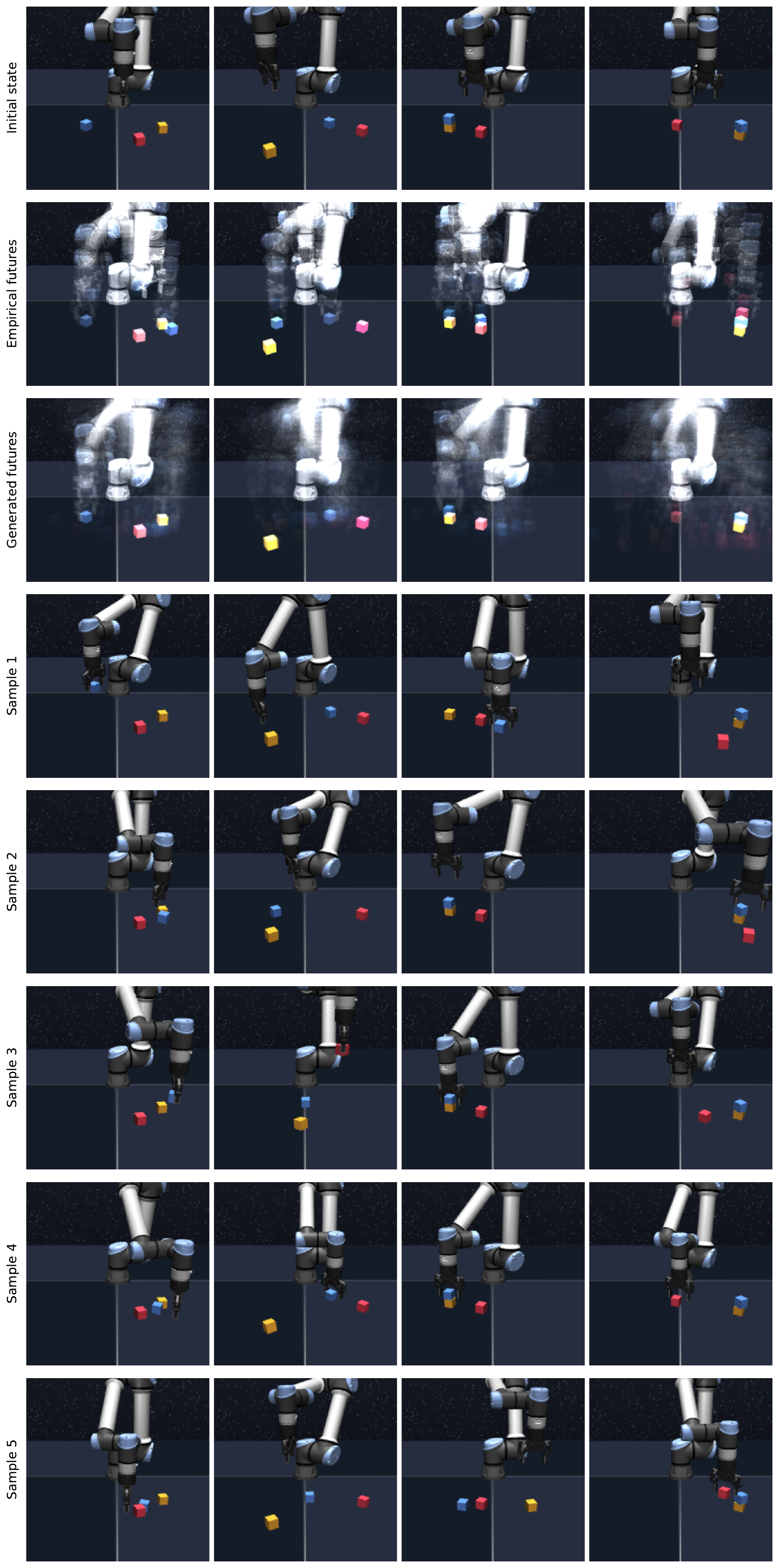}
    \caption{Samples on \texttt{cube-triple-stitch}. The first row shows initial states, and the second shows overlaid empirical futures from the same trajectory. The third row shows overlaid generated samples, while the following rows show individual generated samples. }
    \label{fig:manipsamples}
\end{figure}

\newpage

\FloatBarrier
\clearpage

\begin{table}[H]\centering
\tabcolsep=3pt   
\cmidrulewidth=0.5pt
\fontsize{8}{8}\selectfont
\caption{\textbf{\method with Contrastive RL}. We utilize \method for improving value learning with CRL. We calculate success rate averaged over the last 3 checkpoints, and then report the mean and standard deviation over 3 seeds.}
\label{tab:crl}

\begin{tabular}{lllll}\toprule

{\centering \texttt{pointmaze}} 
& {\centering CRL} 
& {\shortstack{+\methodmc}} 
& {\shortstack{+\methodbase}} 
& {\shortstack{+\methodhc}}
\\\midrule

\texttt{medium-stitch} 
& \pmformat{4}{2} 
& \pmformat{2}{0} 
& \pmformat{24}{6} 
& \pmformat{\mathbf{70}}{6}
\\

\texttt{large-stitch} 
& \pmformat{0}{0} 
& \pmformat{0}{0} 
& \pmformat{4}{5} 
& \pmformat{\mathbf{46}}{17}
\\

\texttt{giant-stitch} 
& \pmformat{0}{0} 
& \pmformat{0}{0} 
& \pmformat{0}{0} 
& \pmformat{\mathbf{1}}{0}
\\

\texttt{teleport-stitch} 
& \pmformat{10}{5} 
& \pmformat{6}{3} 
& \pmformat{6}{6} 
& \pmformat{\mathbf{26}}{10}
\\

\midrule
{average} 
& {4} 
& {2} 
& {9} 
& \textbf{36}
\\

\bottomrule
\end{tabular}
\end{table}

\section{Extension to Offline RL with Contrastive RL }

\label{sec:CRL}
In this section, we apply our data augmentation strategy to Contrastive RL (CRL) \citep{eysenbach2023contrastivelearninggoalconditionedreinforcement}. In particular, instead of augmenting goals for training the policy, we use augmentation to learn better goal-conditioned (on-policy) value functions. CRL in OGBench \cite{park2024ogbenchbenchmarkingofflinegoalconditioned} uses the binary NCE objective to train a critic $f(s,a,g)$:
\begin{align}
\mathcal{L}_{\mathrm{CRL}} = -\mathbb{E}_{\substack{(s,a) \sim \mathcal{D}^\beta,s_- \sim \mathcal{D}^\beta\\  s_+ \sim P_{\mathrm{goal}(\cdot | s,a)} }}\left[\log \sigma (f(s,a,s_+) - \log( 1 - \sigma(f(s,a,s_-))\right]
\end{align}
Here, $s_+ = s_{\rightarrow k}, k \sim \mathrm{geom}(1-\gamma)$, an MC sample of future state, $\sigma$ is the sigmoid function, and the critic is parameterized by $f(s,a,g) = \phi(s,a)^T\psi(g) /\sqrt{d}$. With full rollouts from the mixture policy, the optimal critic estimates a log density ratio with the successor measure $f^*(s,a,g) = \log(m^\beta_\gamma(g | s,a)/p(g))$, and thus, $Q^\beta_\gamma(s,a,g) \propto \exp(f^*(s,a,g))$ for a fixed $g$. However, in practice, standard CRL is subject to the same compositional generalization gap as BC, where we only have truncated trajectories from individual policies, not the mixture policy. Thus, if we can approximately sample goals from the successor measure $M^\beta_\gamma(\cdot | s,a)$ with \methodbf, we can improve the performance of value learning with CRL by replacing $P_{\mathrm{goal}}(\cdot | s,a)$, with $P_{\mathrm{aug}}(\cdot | s,a)$ as in \eqref{eq:paug}. With $\alpha = 0$, we maintain the original CRL objective, and with $ 0 < \alpha \leq 1$, we mix in generated goals with proportion $\alpha$.

Isolating the contribution of \method solely to \textit{value learning}, we train the policy on goals sampled without augmentation for DDPG+BC \citep{fujimoto2021minimalist}:
\begin{align}
\arg\max_\theta \mathbb{E}_{(s,a) \sim \mathcal{D}^\beta, g \sim P_{\mathrm{mixed}}(\cdot | s,a)}\left[ \lambda \log \pi_\theta(a | s,g) + f(s,\pi^\mu_\theta(s,g),g)\right] 
\end{align}
Where $P_{\mathrm{mixed}}(\cdot | s,a)$ denotes a proportion of $0.5$ for random states in the dataset, and $0.5$ for future MC samples, thus only using ground-truth samples. In Table \ref{tab:crl}, we utilize the standard OGBench \citep{park2024ogbenchbenchmarkingofflinegoalconditioned} configuration for CRL, with $\alpha=0.25$ for the augmentation probability when training the contrastive critic. We can see that augmenting goals with \method leads to improved overall performance ($4 \rightarrow 36$). We also apply our visual setup as in Section \ref{sec:visual} also utilizing $\alpha=0.25$, except we reconstruct generated frames with the decoder so that we can perform CRL without modification. Here, we see smaller gains ($56 \rightarrow 61)$. Further work could study if synthetic samples or other data augmentation may allow for improved scaling of CRL in the offline setting, which has been successful previously in the online setting \citep{wang20261000}.

\begin{table}[h]\centering
\tabcolsep=3pt
\cmidrulewidth=0.5pt
\fontsize{8}{8}\selectfont
\caption{\textbf{\method with Contrastive RL in visual environments.}}
\label{tab:visual-crl-mchc}

\begin{tabular}{lll}\toprule

\text{\centering \texttt{visual-antmaze}} 
& \text{\centering CRL} 
& \text{\shortstack{+\methodhc}}
\\\midrule

\texttt{medium-stitch} 
& \pmformat{82}{2} 
& \pmformat{\mathbf{86}}{1}
\\

\texttt{large-stitch} 
& \pmformat{29}{6} 
& \pmformat{\mathbf{36}}{4}
\\

\midrule
\text{average} 
& {56} 
& \textbf{61}
\\

\bottomrule
\end{tabular}
\end{table}

\end{document}